\documentclass[12pt]{article}
\usepackage{enumitem}

\usepackage{graphicx}
\usepackage{epstopdf}
\usepackage{placeins}
\usepackage{subfigure}
\usepackage{algpseudocode}
\usepackage{algorithmicx}
\usepackage[boxruled,algosection]{algorithm2e}

\usepackage{url}
\usepackage[table]{xcolor}
\newcolumntype{L}[1]{>{\raggedright\let\newline\\\arraybackslash\hspace{0pt}}m{#1}}
\newcolumntype{C}[1]{>{\centering\let\newline\\\arraybackslash\hspace{0pt}}m{#1}}
\newcolumntype{R}[1]{>{\raggedleft\let\newline\\\arraybackslash\hspace{0pt}}m{#1}}

\usepackage{amsmath}
\usepackage{amsfonts}
\usepackage{amsthm}
\usepackage{cite}

\usepackage{geometry}

\newtheorem{theorem}{Theorem}
\newtheorem{definition}{Definition}
\newtheorem{example}{Example}

\usepackage[affil-it]{authblk} 
\usepackage{etoolbox}
\usepackage{lmodern}
\makeatletter
\patchcmd{\@maketitle}{\LARGE \@title}{\fontsize{16}{19.2}\selectfont\@title}{}{}
\makeatother

\title{Learning with fuzzy hypergraphs: a topical approach to query-oriented text summarization}
\author{Hadrien Van Lierde and Tommy W. S. Chow\\
	Department of Electronic Engineering, City University of Hong Kong\\
	83 Tat Chee Av., Kowloon Tong, Hong Kong, China\\
	hadrien.vanlierde@hotmail.com, eetchow@cityu.edu.hk\\
	}
\date{}

\begin{document}
\footnotesize\noindent\textit{This is the unrefereed Author's Original Version (or pre-print Version) of the article. The present version is not the Accepted Manuscript. The publication details of the manuscript are the following: H. Van Lierde, T.W.S. Chow, Learning with fuzzy hypergraphs: A topical approach to query-oriented text summarization, \textit{Information Sciences}, 496 (2019), 212-224, \url{https://doi.org/10.1016/j.ins.2019.05.020}.}
{\let\newpage\relax\maketitle}
\maketitle
\begin{abstract}

Existing graph-based methods for extractive document summarization represent sentences of a corpus as the nodes of a graph or a hypergraph in which edges depict relationships of lexical similarity between sentences. Such approaches fail to capture semantic similarities between sentences when they express a similar information but have few words in common and are thus lexically dissimilar. To overcome this issue, we propose to extract semantic similarities based on topical representations of sentences. Inspired by the Hierarchical Dirichlet Process, we propose a probabilistic topic model in order to infer topic distributions of sentences. As each topic defines a semantic connection among a group of sentences with a certain degree of membership for each sentence, we propose a fuzzy hypergraph model in which nodes are sentences and fuzzy hyperedges are topics. To produce an informative summary, we extract a set of sentences from the corpus by simultaneously maximizing their relevance to a user-defined query, their centrality in the fuzzy hypergraph and their coverage of topics present in the corpus. We formulate a polynomial time algorithm building on the theory of submodular functions to solve the associated optimization problem. A thorough comparative analysis with other graph-based summarization systems is included in the paper. Our obtained results show the superiority of our method in terms of content coverage of the summaries.\\
\noindent\textbf{keywords:} Automatic Text Summarization, Fuzzy Graphs, Probabilistic Topic Models, Hierarchical Dirichlet Process, Personalized PageRank, Submodular Set Functions
\end{abstract}


\section{Introduction}\label{introSection}

The rapid expansion of the Internet led to a substantial increase in the amount of publicly available textual resources in recent years. The availability of information in the form of online documents such as news articles or legal texts facilitates decision processes in fields ranging from finance to legal matters. Automatic text summarization speeds up the process of information extraction by automatically producing summaries of large corpora. While early methods were restricted to the summarization of single documents, recent approaches focused on the more realistic problem of multi-document summarization \cite{nenkova2011}. Similarly, the interest has evolved from generic towards query-focused summarizers, which produce summaries with the information relevant to a query formulated by the user.

While an abstractive summarizer generates an abstract of a corpus based on natural language generation, extractive summarizers produce summaries by extracting and aggregating relevant sentences of the original corpora. The large majority of algorithms build on the extractive approach since it focuses on the design of sentence ranking functions that score sentences in terms of relevance and it does not require extensive Natural Language Processing. Among these algorithms, graph-based summarizers have proved to outperform feature-based methods in various experiments \cite{nenkova2011} due to their ability to capture the global structure of connections between sentences of a corpus in the calculation of sentence scores. In their simplest form, graph-based summarizers first define a graph in which vertices are sentences and edges represent pairwise lexical similarities between sentences, namely similarities based on the number of words sentences have in common. Then sentence scores are obtained by applying popular graph-based ranking algorithms such as PageRank \cite{R17} or HITS algorithm \cite{wan2008}. Recently graph-based summarizers were proposed to address the subtask of query-focused summarization. A popular graph-based sentence ranking method to address this problem is the so-called personalized PageRank algorithm which introduces a query bias in the probabilities of transition between sentences and, in turn, scores sentences in terms of both their centrality in the graph and their relevance to the query \cite{R17}. Since a simple graph consisting of pairwise connections among sentences is unable to model complex collective relationships among multiple sentences, hypergraph models were also proposed \cite{wanng2013,xiong2016}, which capture groups of lexically similar sentences and then apply hypergraph extensions of ranking algorithms.

Two limitations of existing graph- and hypergraph-based algorithms alter their summarization capabilities: the \textit{semantic} limitation and the lack of \textit{topical diversity}. First, the calculation of similarities between sentences is generally based on the co-occurrence of terms in sentences (lexical similarity) rather than their \textit{semantic} relatedness \cite{wan2008,R7}. However, two sentences with no or few words in common might still refer to the same topic or have a similar meaning in the context of a specific corpus, as shown by the following example. 

\begin{itemize}
\itemsep0em
\item [--] \textit{After landing, the airplane slowly moved on the track until it stopped at its parking place.}
\item [--] \textit{The aircraft reached a designated area and the passengers got off.}
\end{itemize}

Although they provide slightly different pieces of information, both sentences are semantically related as they share semantically related terms. However, they do not have any word in common, except stopwords. The sentence graph or hypergraph should ideally capture such semantic relationships among sentences. Indeed, since the graph construction has a significant impact on the sentence scores, neglecting semantic relationships among sentences alters the quality of the final summary. Attempts to incorporate higher order relationships among sentences include the detection of clusters of lexically similar sentences, namely groups of sentences with a large number of words in common \cite{wan2008,wanng2013,zhang2012,cai2013}. Although these cluster-level relationships can capture semantic similarities to some extent, they do not attempt to detect sets of semantically related terms or topics. As a result, they fail to capture pairwise semantic similarities between sentences when they use very different wordings, as in the example above.

Second, most systems include a greedy sentence selection method for redundancy removal in which sentences are considered redundant only if they have words in common \cite{xiong2016}. Other methods include methods simultaneously maximizing relevance and minimizing redundancy \cite{yin2015,lin2010} and methods based on the detection of dominating sets \cite{shen2010}. These different approaches build on lexical similarities between sentences as a measure of their redundancy. However, as shown in the example above, lexically dissimilar sentences might still be semantically related. Hence, with existing algorithms of redundancy removal, the resulting summary might consist of sentences that refer to the same topic and fail to cover all major topics of the given corpus. A new approach is thus needed to enforce \textit{topical diversity} in summaries instead of removing lexical redundancies.

To address the \textit{semantic} limitation of existing systems, we propose to capture semantic relationships among sentences making use of a probabilistic topic model called the Hierarchical Dirichlet Process, which was originally designed for the detection of topics in corpora of documents \cite{hdp}. We adapt the model for the inference of sentence topics. The model inference is based on Gibbs sampling. The model infers topics as groups of semantically related terms in the given corpus, and it labels each sentence with multiple topic tags and associated topic weights. Since each topic connects a group of semantically related sentences and since the importance of each topic in a sentence is weighted, we model sentences as a fuzzy hypergraph, namely an extension of hypergraphs in which hyperedges are fuzzy subsets of the set of nodes. In our fuzzy hypergraph model, nodes are sentences, fuzzy hyperedges are topics and the weights of a topic in each sentence define its distribution over vertices. As it involves topical relationships, this fuzzy hypergraph captures the semantic similarities of sentences.

A recent idea proposed in \cite{xiong2016} shares some similarities with our approach as it also incorporates topics inferred by a topic model in a hypergraph-like structure. They cluster sentences based on their topical representations and the resulting disjoint communities are modelled as crisp and disjoint hyperedges of a hypergraph instead of fuzzy hyperedges. Modelling semantic similarities as non-overlapping clusters in such a way fails to capture the multiplicity of topics covered by sentences.

To address the issue of \textit{topical diversity}, we propose a new sentence selection approach based on our fuzzy hypergraph. This approach produces a summary by extracting the sentences maximizing \textit{Relevance} and \textit{Topical Coverage}. The \textit{Relevance} of individual sentences express both their similarity with the query and their centrality in the corpus. Relevance scores are computed through an extension of Personalized PageRank algorithm for our fuzzy hypergraph. The \textit{Topical Coverage} of a set of sentences expresses the multiplicity and diversity of topics covered by these sentences. Our definition of Topical Coverage is based on an extension to our fuzzy hypergraph of dominating set problem \cite{garey2002}. Hence, instead of removing lexical redundancies, we intend to improve the topical diversity of our summary, which is more consistent with the goal of covering all major topics of a given corpus. Relevance and Topical Coverage are combined into a discrete optimization problem for sentence selection. As the problem is shown to be NP-hard, we formulate an approximation algorithm with a relative performance guarantee. The algorithm is based on the theory of submodular functions. This core algorithm of sentence selection is called \textit{Maximum Relevance and Coverage} (MRC) algorithm. The final summary is obtained by aggregating the selected sentences.

The main contributions of this paper are the following: (1) a new fuzzy hypergraph model capturing semantic relationships among sentences of a corpus inferred by a probabilistic topic model, (2) a multi-objective optimization problem expressing the sentence selection process as the maximization of Relevance of individual sentences and Topical Coverage of the resulting summary and (3) a polynomial time algorithm building on the theory of submodular functions for solving the optimization problem and generating informative and semantically diverse summaries.

The structure of the paper is as follows. In section \ref{relatedWorks}, we present summarization algorithms related to ours. In section \ref{overallSection}, we present an overview of our system. In section \ref{mainSection}, we present each step of our framework including the topic modelling, the fuzzy hypergraph construction and the sentence selection. Finally, in section \ref{experimentSection}, we present experimental results demonstrating the superiority of our approach over state-of-the-art summarizers on real-world datasets. 

\section{Related work}\label{relatedWorks}
Extractive summarizers aggregate important sentences in a corpus while abstractive summarizers generate new summaries after identifying important information \cite{nenkova2011}. As abstractive summarization requires extensive Natural Language Processing, most summarizers to date are based on extractive approaches.

Methods of extractive summarization generally fall into two categories, namely feature-based and graph-based approaches. Feature-based methods train a model to predict the score of each sentence based on feature representations of sentences (term frequency, sentence position \cite{nenkova2011}, etc.). Graph-based methods define graphs in which nodes are sentences and edges represent similarities between sentences. Sentence scores are then given by node centrality measures on the graph \cite{R17,R7}. The advantages of graph-based summarization over feature-based summarization are that it does not require labelled corpora, and it is based on the global structure of links between sentences of the corpus rather than local features.

The earliest graph-based summarizer, called LexRank \cite{R7}, defines edges as term co-occurrence relationships between sentences. Then, PageRank algorithm is applied to compute relevance scores of sentences. Adapting this idea for the task of query-focused summarization, topic sensitive LexRank \cite{R17} introduces a query bias in probabilities of transition, which results in higher scores for sentences that are similar to the query. Similarly, \cite{wan2013} proposes a manifold ranking algorithm in which scores are popagated accross a graph including both sentences and the query as vertices. To remove redundancies in summaries, \cite{mei2010} proposes a new node ranking algorithm called DivRank, which tends to select dissimilar sentences. While early graph-based algorithms only involved sentences, a bipartite graph model is proposed in \cite{wan2008}, involving both sentences and terms as vertices and it applies HITS algorithm to score sentences. \cite{R30} combines this idea with a PageRank-like method to score sentences, terms and documents simultaneously.

While early methods build sentence graphs based on co-occurrence of terms in sentences only, later approaches infer higher level relationships. These methods include sentence clusters in the graph construction, namely groups of similar sentences. In that perspective, \cite{wan2008} builds a bipartite graph in which vertices consist of both sentences and clusters, and edges represent lexical similarities between sentences and clusters. HITS algorithm is applied to score both sentences and clusters. A similar idea presented in \cite{zhang2012} incorporates terms as a third class of vertices. While these algorithms only discover clusters of lexically coherent sentences using standard clustering algorithms, \cite{cai2013} suggests that scores of sentences within each community should be quite different from each other. Wang et al. presents an alternative way to incorporate higher level connections among sentences \cite{wanng2013}: they build a hypergraph in which nodes are sentences and hyperedges represent clusters. Then, sentence scores are computed based on semi-supervised learning over hypergraphs. Although this hypergraph models relationships that are more complex than pairwise, their method is limited to disjoint sentence clusters which results in binary and non-overlapping hyperedges. Hence, the hypergraph poorly models the multiplicity of topical relationships among sentences.

In contrast, several summarizers propose to build on topic models rather than clusters, namely to infer a set of topics for a given corpus, each topic being modelled as a distribution over terms. When applied in the context of text summarization, each sentence is tagged with multiple topics, which better models the multiple information carried by sentences. Popular topic modelling algorithms include Probabilistic Latent Semantic Analysis (PLSA), Latent Dirichlet Allocation (LDA) and the Hierarchical Dirichlet Process (HDP). \cite{Hennig2009} computes the similarity of sentences with a user-defined query based on PLSA. Going beyond PLSA, \cite{arora2008} extracts topic distributions of sentences based on LDA and, for each topic, it selects the sentence with highest associated probability. While LDA overcomes PLSA's tendency to overfit by setting a Dirichlet prior on the distribution of documents over topics, the number of topics must be determined by cross validation. In contrast, the topic model present in our system is based on HDP, which automatically infers the number of topics by incorporating Dirichlet Processes as nonparametric priors for topics \cite{hdp}. Moreover, the hierarchical structure of HDP allows us to infer both sentence and document topics simultaneously.

A hypergraph model similar to ours was presented recently in \cite{xiong2016} which uses HDP to compute sentence embeddings. Then, sentence clusters are extracted by applying a standard clustering algorithm to these sentence embeddings. These non-overlapping clusters define binary and disjoint hyperedges that do not capture the multiplicity of topics covered by a sentence, which can only be captured by overlapping and fuzzy hyperedges as the ones present in our model.

Building on fuzzy set theory, fuzzy graphs associate each node with a degree of membership in each edge \cite{fuzzy}. Relaxing the assumption of pairwise relationships, fuzzy hypergraphs are defined by a set of nodes and a set of fuzzy subsets of these nodes. Applications of fuzzy hypergraphs include portfolio management and managerial decision making \cite{fuzzy,bershtein2009}. To our knowledge, fuzzy hypergraphs have not yet been used for text mining purposes, including text summarization. Fuzzy hypergraphs are used to incorporate topical information in our summarizer.

After sentence scoring, a critical step is to select highly scored sentences that are not redundant. A popular method is the greedy method of redundancy removal which selects dissimilar sentences with highest scores \cite{xiong2016}. As this method may favour long sentences, multi-objective approaches were proposed in order to maximize the sum of relevance scores of selected sentences and simultaneously minimize their redundancy \cite{yin2015,lin2010}. However, their definition of redundancy is limited to lexical similarities. Other methods include the one in \cite{shen2010}, which selects sentences by solving the dominating set problem over the sentence graph. However, their algorithm also tends to favour long sentences over short ones and it fails to model semantic relationships captured by topics. In general, existing methods of redundancy removal are merely based on lexical similarities between sentences which does not prevent semantic redundancies in the final summary. In contrast, our approach based on Topical Coverage selects sentences covering the main topics of the corpus, which automatically reduces their semantic redundancy.

\section{Problem statement and system overview}\label{overallSection}
The problem we intend to solve is that of query-oriented multi-document summarization, namely the production of a summary containing the most important information found in a given corpus and that is also relevant to a user-defined query. This is done by extracting and aggregating relevant sentences from the corpus. We provide a definition of the query-oriented summarization task.

\begin{definition}[Query-oriented summarization problem]
Given a corpus of documents consisting of a set $V$ of sentences, the set $\{l(s):s\in V\}$ of sentence lengths, a summary capacity $L>0$ and a query represented by a sentence $q$, produce a summary $S$ in which $S\subseteq V$ is a set of selected sentences that are relevant to $q$ and contain the essential information of $V$, such that the capacity constraint $\underset{s\in S}{\sum}l(s)\leq L$ is satisfied. 
\end{definition}

\begin{figure}[!h]
\centering
\includegraphics[width=0.95\textwidth]{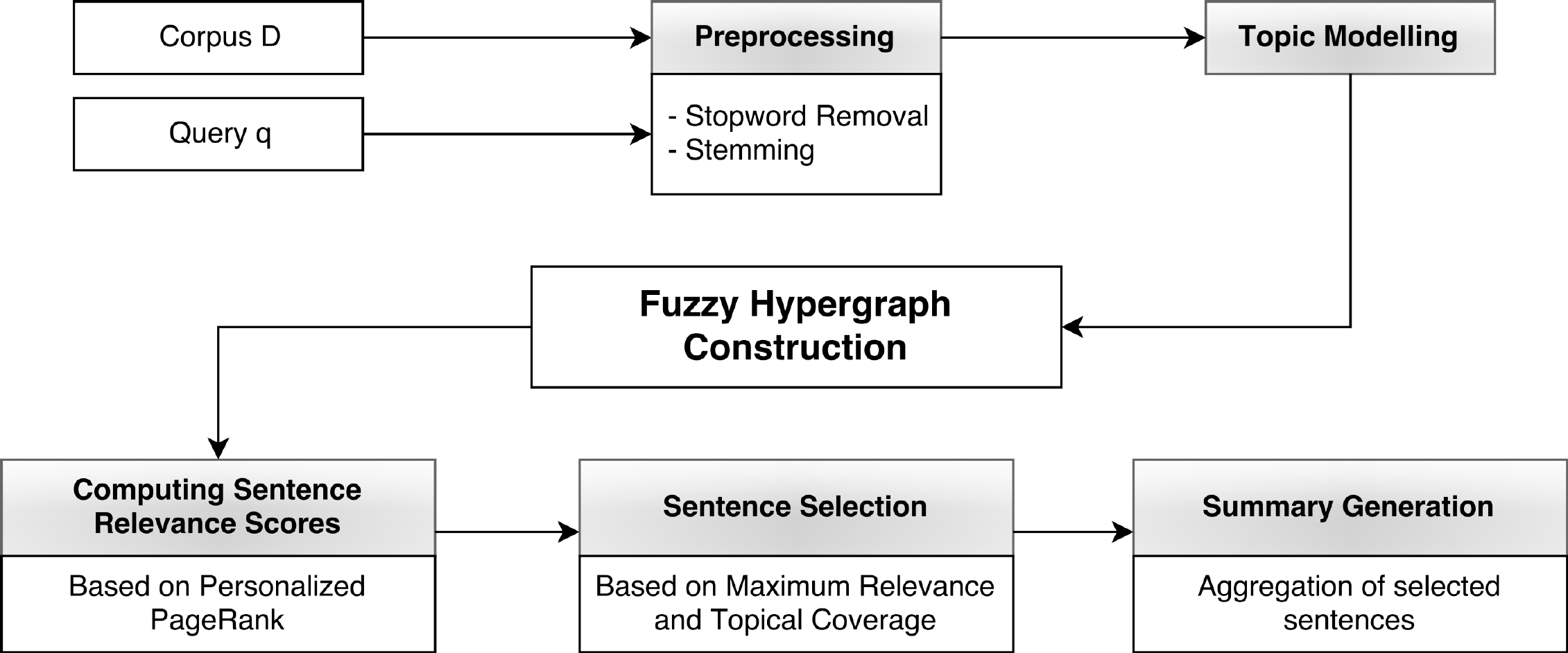}
\caption{System Chart}
\label{systSummary94}
\end{figure}

Hence, we refer to a summary as the set $S$ of selected sentences. The prescribed summary length is the so-called \textit{capacity} of the summary. Our MRC algorithm consists of the following steps which are summarized in figure \ref{systSummary94}.

\begin{enumerate}
\itemsep0em
\item Preprocessing: standard preprocessing steps for sentence vectorization,
\item Topic detection based on the Hierarchical Dirichlet Process,
\item Fuzzy hypergraph definition in which nodes are sentences and fuzzy hyperedges are defined by topics,
\item Computation of sentence relevance scores based on a PageRank-like algorithm over the fuzzy hypergraph followed by the selection of sentences through the maximization of Relevance and Topical Coverage,
\item Generation of the summary by aggregating the selected sentences.
\end{enumerate}

In subsequent sections, we refer to the set of terms of a corpus as the set of distinct words appearing at least once in the corpus. 

\section{Maximizing Relevance and Topical Coverage based on a sentence fuzzy hypergraph}\label{mainSection}

We describe each step of our MRC algorithm in details, including preprocessing, topic modelling, fuzzy hypergraph construction and sentence selection through the maximization of sentence relevance and topical coverage.

\subsection{Preprocessing}\label{preprocSection}

We apply standard preprocessing methods in text mining including stopword removal based on a publicly available list of $153$ English stopwords \cite{stopWords} and word stemming using Porter Stemmer \cite{porter2001}. We let $N_t$ represent the number of distinct terms in the corpus after these preprocessing operations are completed.

\subsection{Topic inference}\label{topicSection}
As mentioned in sections \ref{introSection} and \ref{relatedWorks}, traditional graph-based summarization algorithms only take into account the co-occurrence of terms between sentences. However, in order to capture the semantic similarity between sentences, we must go beyond term co-occurrences and capture the information overlap between sentences. This can be done by extracting the different topics present in the corpus and incorporating topical similarities between sentences. In the field of text mining, a topic is a set of terms referring to the same subject in the context of a document or a corpus. Topic inference refers to the joint tasks of discovering these sets of related terms and inferring topic tags for textual units (documents, sentences or words). For instance, the following sentences refer to semantically related objects (pastures and meadows) although they have few words in common.

\begin{example}\label{meadowsExample}
\normalfont
Definitions of pastures and meadows in Cambridge Dictionary \cite{cambridge}:
\textit{
\begin{enumerate}
\itemsep0em
\item A pasture consists of grass or similar plants suitable for animals such as cows and sheep to eat, or an area of land covered in this,
\item Meadows are fields with grass and often wild flowers in them.
\end{enumerate}}
\end{example}

Both sentences in example \ref{meadowsExample} cover a topic related to \textit{nature} or \textit{countryside} and they could be considered as semantically similar. Existing methods of topic inference are generally based on the detection of terms that consistently occur together in the same documents within the corpus. Such sets of terms are considered as referring to the same topic. Previous attempts to incorporate topical information in automatic text summarization were generally based on methods of matrix factorization such as latent semantic analysis (LSA), which lacks the ability to discover interpretable topics, or its probabilistic version (PLSA), which inevitably leads to overfitting \cite{hdp}. More recent probabilistic topic models describe the process of generation of documents from topics represented as distributions over terms. Among these methods, Latent Dirichlet Allocation was already used for the purpose of summarization. However, a major drawback of this method is the necessity of selecting the number of topics manually. Hence, we rather rely on the Hierarchical Dirichlet Process, which is a probabilistic topic model that is capable of inferring the number of topics automatically.

The Hierarchical Dirichlet Process (HDP) is a mixture model with hidden number of components that builds on the Dirichlet Process (DP). The Dirichlet Process itself can be viewed as a distribution over a set of discrete probability measures with infinite support \cite{hdp} which verifies the following property
\begin{equation}\label{eqnStickBreak}
\text{If }G\sim DP(\gamma,H)\text{ then, with probability 1, }G=\sum_{k=1}^{\infty}\beta_k\delta_{\phi_k}
\end{equation}
where $\gamma$ is a positive parameter, $H$ is a prior distribution on components, $\beta_k$'s are the so-called stick breaking weights and $\phi_k$'s are atoms drawn from $H$. Hence, the Dirichlet Process can be viewed as a measure on measures which extracts a countable infinite number of atoms from a prior distribution. In the context of topic modelling of documents, $H$ is selected to be a $N_t$-dimensional Dirichlet distribution and a draw $G$ of $DP(\gamma,H)$ extracts a countable infinite set of $N_t$-dimensional probability vectors $\phi_k$. Each $\phi_k$ is a vector of probabilities over terms which can be viewed as a topic. 

The original version of HDP is a generative model meant to infer topics of documents within a corpus. Given a set of $N_D$ documents consisting of $N_t$ distinct terms, each document $j$ is represented as a sequence of $n_j$ words $w_{1j},...,w_{n_jj}$ drawn from the $N_t$ terms. The goal is to infer a finite number $K$ of topics in the form of probability distributions over terms $\phi_1,...,\phi_K\in [0,1]^{N_t}$, and a topic tag $z_{lj}\in \{1,...,K\}$ for each word $l$ in document $j$. HDP models the generation of each word from hidden topic vectors $\{\phi_1,...,\phi_k\}$ in the following way.

\begin{enumerate}
\itemsep0em
\item Draw a global measure at a corpus-level $G_0|\gamma,H\sim DP(\gamma,H)$, where the prior distribution $H$ is often chosen as a $N_t$-dimensional symmetric Dirichlet distribution $\text{dir}(\zeta\frac{\mathbf{1}_{N_t}}{N_t})$ in which $\mathbf{1}_{N_t}$ is the $N_t$-dimensional vector of ones. This distribution is the conjugate prior of the categorical distribution and allows a straightforward inference based on Gibbs sampler \cite{hdp}. Parameter $\gamma$ commands the lever of variability of $G_0$ around prior $H$.
\item For $j$-th document:
\begin{enumerate}
\itemsep0em
\item draw a document-specific measure $G_j|\alpha,G_0\sim DP(\alpha,G_0)$,
\item for word $l$ in document $j$:
\begin{itemize}
\itemsep0em
\item draw a distribution over terms $\theta_{lj}|G_j\sim G_j$,
\item draw a term $w_{lj}|\theta_{lj}\sim \text{cat}(\theta_{lj})$, where $\text{cat}(\theta_{lj})$ is the categorical distribution over terms, with probabilities given by vector $\theta_{lj}$.
\end{itemize}
\end{enumerate}
\end{enumerate}
Each draw from a Dirichlet Process extracts a countable infinite number of atoms from a base distribution. Starting from a Dirichlet distribution $H=\text{dir}(\zeta\frac{\mathbf{1}_{N_t}}{N_t})$, global measure $G_0$ draws a countable infinite set $S_0$ of vectors of probability distribution over terms. Each document-level measure $G_j$ extracts a subset $S_j$ of $S_0$. Finally each word $l$ of document $j$ first draws a vector probabilities over terms $\theta_{jl}$ from set $S_j$, and then a term $w_{lj}\sim \text{cat}(\theta_{lj})$. The weight associated to each atom of global and document-level measures is given by stick-breaking weights \cite{hdp}, as suggested by equation \ref{eqnStickBreak}. Due to the discrete nature of measures $G_0$ and $G_j$, distributions $\theta_{lj}$ are naturally shared within documents and within the corpus, with several words being associated to the same distribution over terms. The extent to which the term distributions are shared within documents is commanded by concentration parameter $\alpha$, and within the whole corpus by concentration parameter $\gamma$. The larger these parameters, the larger the number of distinct vectors $\theta_{lj}$ that are generated. In order to extract topic representations and topic tags one can extract the set $\{\phi_1,...,\phi_K\}$ of $K$ distinct vectors $\theta_{lj}$ across all documents and relabel topic tags such that word $l$ in document $j$ is associated to a tag $z_{lj}\in\{1,...,K\}$ so that $\theta_{lj}=\phi_{z_{lj}}$. It is important to note that, opposite to LDA, the number $K$ of topics is inferred and it is not a parameter of the model. Once topic representations are learnt on a training set, topic tags for new previously unseen documents can be predicted.

The model inference can be done by first sampling topic tags and topic representations based on Gibbs sampler and then by extracting Maximum A Posteriori estimators of topic representations $\{\phi_e:1\leq e\leq K\}$ and topic tags $\{z_{lj}:1\leq j\leq N_D,1\leq l\leq n_j\}$ \cite{hdp}. If a Dirichlet distribution is chosen as a prior $H$, which is conjugate to the categorical distribution, Gibbs sampling equations are derived in a straightforward way based on the Chinese Restaurant Franchise model presented in \cite{hdp}. Finally, the set of topic labels $\{z_{lj}:1\leq l\leq n_j\}$ for document $j$ can be interpreted as a set of topic tags and the semantic similarity between two documents can be computed based on the number of topics they have in common.

\begin{figure}[!h]
\centering
\includegraphics[width=.8\textwidth]{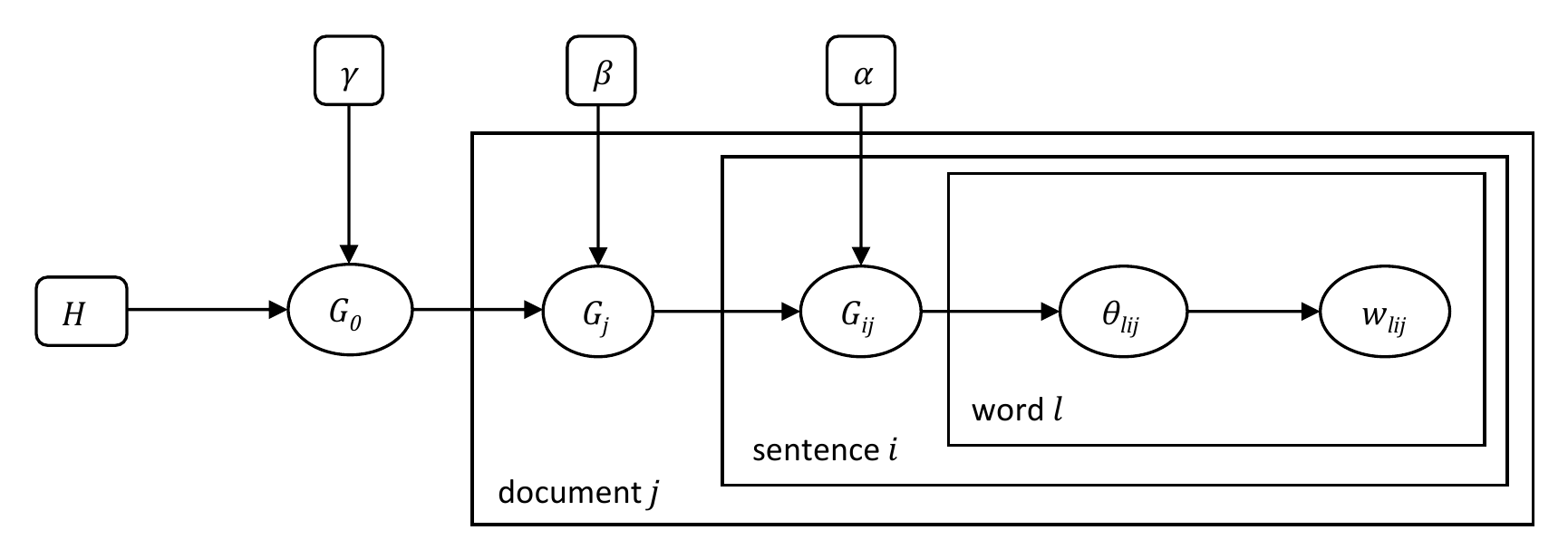}
\caption{Representation of two-level Hierarchical Dirichlet Process.}
\label{2levelHDP}
\end{figure} 

In the context of graph-based extractive text summarization, since we are interested in the computation of semantic similarities between sentences, we need to extract topic tags for sentences instead of entire documents. Several previous studies \cite{arora2008,xiong2016} proposed to do so by first extracting topic tags $\{z_{lj},1\leq l\leq n_j\}$ for each document $j$ using HDP or LDA and then, for a sentence consisting of the subsequence of words $w_{l_1j},...,w_{l_sj}$ of document $j$, topic tags of the sentence are given by the corresponding subset of tags $\{ z_{l_1j},...,z_{l_sj}\}$ \cite{xiong2016,arora2008}. However, as can be seen from our description of HDP model above, an important assumption of it is the so-called \textit{exchangeability} assumption which neglects word ordering in documents. Documents are thus regarded as bag-of-words. Due to this exchangeability assumption, the partitioning of words into sentences is not taken into account in the model. Hence, merely defining sentence topic tags as a subset of document tags neglects the topical information jointly carried by words within a sentence. It is thus not guaranteed to produce coherent topic tags for each sentence. In particular, all words within a sentence could be assigned different topics. It would thus be desirable to also encourage topics to be shared by words within a sentence, in order to properly capture the semantics of a sentence as a whole. This can be done by extending the model above with a two-level HDP, both at the document and sentence levels, as depicted in figure \ref{2levelHDP}. In our model, the process of generation of word $l$ in sentence $i$ of document $j$ is as follows.

\begin{enumerate}
\itemsep0em
\item Draw a global measure $G_0|\gamma,H\sim DP(\gamma,H)$,
\item for each document $j$:
\begin{enumerate}
\itemsep0em
\item draw a document-specific measure $G_j|\beta,G_0\sim DP(\beta,G_0)$,
\item for each sentence $i$ in document $j$:
\begin{itemize}
\itemsep0em
\item draw a sentence-specific measure $G_{ij}|\alpha,G_j\sim DP(\alpha,G_j)$,
\item for each word $l$ in sentence $i$ of document $j$:
\begin{itemize}
\itemsep0em
\item draw a distribution over terms $\theta_{lij}\sim G_{ij}$,
\item draw a term $w_{lij}\sim \text{cat}(\theta_{lij})$.
\end{itemize}
\end{itemize}
\end{enumerate}
\end{enumerate}

The two-level HDP ensures that topics are shared across documents, across sentences and within sentences. In such context, the closer two words are in the hierarchy of corpus, documents and sentences, the more likely they are to fall in the same topic. Such two-level HDP was already proposed in another context, namely for the inference of topic tags for documents belonging to several corpora, with draws from Dirichlet Processes both at corpus and at document level. However, to the best of our knowledge, this model was not proposed for the inference of sentence topics within a corpus of documents. For the inference of topic tags and topic distributions over terms in our two-level HDP, the Gibbs sampler of \cite{hdp} is used, in which such two-level HDP is also introduced with both corpus and document levels. We refer to \cite{hdp} for the set of sampling equations for the inference of topic tags and topic distributions over terms based on a Markov Chain Monte Carlo algorithm. The computational complexity of each pass of Gibbs sampling algorithm for HDP is proportional to the corpus length. In practical applications involving hundreds of documents of the size of a newspaper article, the convergence of the algorithm is fast compared to the computation of sentence relevance scores presented next \cite{wang2012}.

After completing the inference, we obtain the quantities below:

\begin{itemize}
\itemsep0em
\item a number $K$ of topics represented by distributions over terms: for $1\leq e\leq K$, the distribution of topic $e$ over terms is
\begin{equation}\label{probaTopicTerm}
\phi_e\in [0,1]^{N_t}
\end{equation}
where $N_t$ is the number of distinct terms in the corpus and $\phi_{et}$ is the probability of observing term $t$ under topic $e$;
\item for $1\leq j\leq N_s$, $1\leq i\leq n_j$, the topic tag of word $l$ in sentence $i$ of document $j$ is represented by variable $z_{lij}\in\{1,...,K\}$.
\end{itemize}

As we choose $H$ to be a Dirichlet prior $\text{dir}(\zeta\frac{\mathbf{1}_{N_t}}{N_t})$, there are four dispersion parameters in this topic modelling step, namely $\alpha$, $\beta$, $\gamma$ and $\zeta$. Experiments presented in section \ref{experimentSection} estimate suitable ranges of values for these parameters.

\subsection{Fuzzy hypergraph definition}\label{fuzzygraphSection}
A hypergraph $H=(V,E)$ over a set $V$ of vertices is a generalization of graph in which each hyperedge in $E$ is a subset of $V$ \cite{wanng2013}. In existing hypergraph-based summarizers \cite{wanng2013,xiong2016}, vertices are sentences and clusters of sentences correspond to hyperedges which do not overlap. There is also no attempt to model the degree of membership of each sentence in each hyperedge. This model is unsatisfactory since each sentence may cover multiple topics, and each topic is covered by a sentence with a different degree depending on the number of words of the sentence tagged with this topic. To overcome these limitations, we model sentences as a fuzzy hypergraph, namely a generalization of hypergraph in which hyperedges are defined as fuzzy subsets of the set of nodes. Fuzzy hypergraphs provide accurate models of networks in which agents participate in each connection with a certain degree \cite{fuzzy}. A formal definition of fuzzy hypergraph is given below\footnote{This definition of fuzzy hypergraph is an adaptation of the one in \cite{fuzzy}, in which the degrees of membership of vertices in a hyperedge are normalized to represent a distribution over vertices.}.  

\begin{definition}[Fuzzy Hypergraph]
A fuzzy hypergraph is defined as a quadruplet $G=(V,E,\psi,w)$ on a set $V$ of vertices and a set $E$ of hyperedges such that
\begin{itemize}[leftmargin=*]
\itemsep0em
\item [--] $\psi\in [0,1]^{|E|\times |V|}$ is a matrix that defines a distribution over vertices for each of the $|E|$ hyperedges, verifying
$\underset{i\in V}{\sum}\psi_{ei}=1\text{ for }e\in E$ and $\underset{e\in E}{\sum}\psi_{ei}>0\text{ for }i\in V$,
\item [--] a positive weight $w(e)\in\mathbb{R}^+$ for each hyperedge $e\in E$.
\end{itemize}
\end{definition}

By analogy with the non-fuzzy case, matrix $\psi$ defines the incidence matrix of the fuzzy hypergraph. Each hyperedge defines a group relationship among nodes while the fuzziness of hyperedges allows to quantify the implication of each node in the relationship. In the context of our summarization method, we define a fuzzy hypergraph $G=(V,E,\psi,w)$ in which vertices are sentences and each fuzzy hyperedge represents a topic. The degree of membership of each sentence in a hyperedge is proportional to the number of words tagged with the corresponding topic in the sentence, namely
\begin{equation}
\psi_{ei}=\frac{|\{ l:z_{li}=e\}|}{|\{l:z_{lj}=e,1\leq j\leq N_s\}|}.
\end{equation}
For simplicity, we dropped document index $j$ and we denote by $z_{li}$ the topic of $l$-th word in $i$-th sentence. Unlike previous hypergraph-based approaches, we make the more realistic assumption that each sentence can belong to different semantic groups (i.e. topics) with a certain degree of membership in each group. Example \ref{exampleMultipleTopics} shows a sentence that refers to two topics. The sentence is thus semantically related to any other sentence referring to either topic.
\begin{example}\label{exampleMultipleTopics}
\normalfont The following sentence combines two distinct topics, the topic of \textit{studies} ("homeworks", "school", "exams") and the topic of leisure ("friends", "park", "football", "played"): \textit{"After he finished his homeworks and got prepared for his school exams, the boy met with his friends in the park and they played football."}
\end{example}

Next, we define the weight $w(e)$ of a fuzzy hyperedge $e$ based on the discriminatory power of terms present in the corresponding topic, which depends on four aspects described below. These four term-based factors along with a factor measuring the relevance of topics within the corpus are combined to form hyperedge weights. This method differs from earlier models in which cluster weights were given by their lexical similarity with the entire corpus.

The \textit{in-corpus frequency} $\text{tfc}(t)$ of term $t$ in the corpus is the number of times term $t$ appears in the corpus. The \textit{sentence discriminatory power} $\text{isf}(t)$ of term $t$ is given by the logarithm of the inverse sentence frequency, as proposed in \cite{blake2006}
\begin{equation}\label{isfDef}
\text{isf}(t)=\log\left(\frac{N_s}{N_s^t}\right)
\end{equation}
where $N_s$ is the total number of sentences and $N_s^t$ is the number of sentences containing term $t$. Similar to $\text{idf}$ term weighting \cite{blake2006}, $\text{isf}$ weight is based on the idea that a term occurring in a large number of sentences carries less discriminatory information for the selection of the most relevant sentences. The \textit{in-topic frequency} $\text{tft}(t,e)$ of term $t$ in topic $e$ is the probability of encountering term $t$ conditioned on $e$ which is computed in the HDP inference process (equation \ref{probaTopicTerm}), i.e.
\begin{equation}
\text{tft}(t,e)=\phi_{et}.
\end{equation}
The \textit{topic discriminatory power} $\text{tdp}(t)$ of a term $t$ is based on the idea that a term $t$ appearing in relatively few topics should have a significant contribution to the semantics of sentences and topics while terms appearing in a large number of topics might have ambiguous meanings. We quantify the topic discriminatory power of a term $t$ by measuring the entropy of its distribution over topics:
\begin{equation}\label{EntropyOneEqn}
H(t)=-\underset{e}{\sum}p(e|t)\log(p(e|t))
\end{equation}
where $p(e|t)$ measures the fraction of occurrences of term $t$ in the corpus that are tagged with topic $e$. Then, the topic discriminatory power of $t$ is given by a shifted inverse of the entropy of this distribution
\begin{equation}
\text{tdp}(t)=\frac{1}{1+H(t)}
\end{equation}
which is equal to $1$ if $t$ is only tagged with a single topic in the whole corpus.

Finally the \textit{relevance} $\text{rel}(e)$ of topic $e$ is computed as
\begin{equation}
\text{rel}(e)=\text{f}(e)\log\left( \frac{N_s}{N_s^e}\right)
\end{equation}
where $\text{f}(e)$ is the number of occurrences of topic $e$ in the corpus and $N_s^e$ is the number of sentences in which topic $e$ occurs. The relevance $\text{rel}(e)$ of topic $e$ can be viewed as an adaptation of the term-frequency-inverse-sentence-frequency (tfisf) weights for weighting topics instead of terms \cite{blake2006}.

\begin{algorithm}[H]~\\
INPUT: for $1\leq e\leq K$, $\phi_e$, $N^e_s$ (number of sentences tagged with topic $e$) and $f(e)$; for $1\leq t\leq N_t$, $N_s^t$ (number of sentences containing term $t$); for $1\leq i\leq N_s$, topic tags $z_{wi}$ for each word $w$ in sentence $i$,\\
OUTPUT: Hypergraph $H(\{1,...,N_s\},\{1,...,K\},\psi,w)$\\
\textbf{for each} $t\in\{1,...,N_t\}$:\\
	\Indp Compute $\text{tfc}(t)$ and $H(t)$ (equation \ref{EntropyOneEqn})\\
	Let $\text{isf}(t)\leftarrow\log\left(\frac{N_s}{N_s^t}\right)$ and $\text{tdp}(t)\leftarrow\frac{1}{1+H(t)}$\\
	\textbf{for each} $e\in\{1,...,K\}$: $\text{tft}(t,e)\leftarrow\phi_{et}$\\
\Indm\textbf{for each} $e\in\{1,...,K\}$:\\
	\Indp $\text{rel}(e)\leftarrow\text{f}(e)\log\left( \frac{N_s}{N_s^e}\right)$\\
	$w(e)\leftarrow\text{rel}(e)\underset{t}{\sum} \text{tfc}(t)\text{isf}(t)\text{tft}(t,e)\text{tdp}(t)$\\
	\textbf{for each} $i\in\{1,...,N_s\}$: $\psi_{ei}\leftarrow\frac{|\{ l:z_{li}=e\}|}{|\{l:z_{lj}=e,1\leq j\leq N_s\}|}$\\
\caption{Fuzzy Hypergraph Construction}
\label{algorithmHYPER}
\end{algorithm}
\FloatBarrier

The weights of hyperedges are obtained by combining all the above scores:
\begin{equation}
w(e)=\text{rel}(e)\underset{t}{\sum} \text{tfc}(t)\text{isf}(t)\text{tft}(t,e)\text{tdp}(t).
\end{equation}
This definition yields a high weight for frequent topics including terms that occur a large number of times in the corpus, have strong discriminatory power over sentences and are not semantically ambiguous. As opposed to previous topic-based summarization algorithms \cite{xiong2016,arora2008}, we take advantage of the representation of topics as distributions over terms in order to compute the topic weights. Algorithm \ref{algorithmHYPER} summarizes the step of the fuzzy hypergraph construction. The computational complexity of the algorithm is $O(K(N_s+N_t))$ where $K$ is the number of topics.

\subsection{Relevance and Coverage Maximization for sentence selection}\label{scoreSection}
We present the consecutive steps of sentence scoring and selection. Based on the fuzzy hypergraph defined previously, we rank each sentence in terms of its relevance to the query and its centrality in the whole corpus. Then, we select a set of sentences maximizing individual Relevance and joint Topical Coverage.

\subsubsection{Computing relevance scores of sentences}
We introduce a ranking algorithm that computes scores for sentences according to their relevance to the user-defined query and their centrality in the corpus. Graph-based summarization algorithms rely in general on variations of PageRank algorithm for sentence ranking \cite{R7,R17}. The underlying assumption is that the generation of a coherent text from isolated sentences can be modelled as a Markov chain in which states are sentences and the probability of transition between two sentences depends on their similarity in some sense. Stationary probabilities provide the sentence ranks in the context of generic summarization. We extend this method by defining a random walk over fuzzy hypergraphs in which the transition probability between two vertices depends on the hyperedges shared by these vertices. The transition from vertex $i$ to another vertex is performed in two steps:
\begin{enumerate}
\itemsep0em
\item draw a hyperedge $e\in E$ with probability
$p(e|i)=\frac{p(i|e)w(e)}{\underset{f}{\sum}p(i|f)w(f)}=\frac{\psi_{ei}w(e)}{\underset{f}{\sum}\psi_{fi}w(f)}$,
\item draw a vertex $j$ in $V$ with probability $p(j|e)=\psi_{ej}$.
\end{enumerate}
Integrating out the hyperedges, we obtain the probability of transition
\begin{equation}\label{sentenceTransition}
p(j|i)=\underset{e}{\sum}p(j|e)\frac{p(i|e)w(e)}{\underset{f}{\sum}p(i|f)w(f)}
\end{equation}
from vertex $i$ to vertex $j$. The interpretation of this Markov chain over sentences is the following. Our goal is to generate a coherent sequence of sentences $s(1),s(2),...$ where $s(\tau)$ is the sentence produced by the Markov chain at time step $\tau$. By coherence, we mean that two consecutive sentences must be semantically related. The above transition between two sentences depends on two factors: first the co-occurrence of topics and the degree of membership of each sentence in the corresponding topics, and second the weight of the co-occurring topics.

With the above transition probabilities, the scores of sentences are the stationary probabilities computed by PageRank algorithm. However, as we intend to extract sentences that are both central in the corpus and relevant to a user-defined query, we adapt the formula proposed in \cite{R17} for query-focused text summarization. Given a measure of the probability of transition $p(j|q)$ from the query sentence $q$ to any sentence $j$, the query-biased probability of transition from $i$ to $j$ is
\begin{equation}\label{queryBiasedTransition}
p^q(j|i)=(1-\lambda)p(j|q)+\lambda p(j|i)
\end{equation}
where $\lambda\in [0,1]$ is called the \textit{query balance}, which commands the extent to which scores are learnt from the query relevance or from the propagation of scores across the fuzzy hypergraph. Transition probability $p^q(j|i)$ favours sentences that are similar to the query at each step of the Markov chain, where the query similarity is defined by $p(j|q)$. Equation \ref{sentenceTransition} cannot be used to compute the query relevance term $p(j|q)$, since it would require to infer topics for a potentially short query. To address this issue, we define the following query relevance measure:
\begin{equation}
p(j|q)=\underset{t}{\sum}\underset{e}{\sum} \psi_{ej}p(e|t)p(t|q)
\end{equation}
where $(\psi_{ej})_{\substack{1\leq e\leq K\\1\leq j\leq N_s}}$ is the  incidence matrix of the fuzzy hypergraph as defined in section \ref{fuzzygraphSection}, $p(e|t)$ measures the fraction of occurrences of term $t$ that are tagged with topic $e$ and $p(t|q)$ is the frequency of term $t$ in the query. With such query bias, sentences that are semantically similar to the query get increased probabilities of transition from other sentences, which ultimately results in higher scores for these sentences. This query relevance measure goes beyond the lexical similarity that is generally used in other systems \cite{wanng2013,xiong2016}. The final scores $\{p(i):1\leq i\leq N_s\}$ are obtained by applying PageRank iterative algorithm:
\begin{equation}\label{pagerank}
p^T(j)=(1-\mu)\frac{\mathbf{1}_{N_s}}{N_s}+\mu\sum_{\substack{i=1\\ i\neq j}}^{N_s}p^q(j|i)p^{T-1}(i)\text{, }T=1,2,...
\end{equation}
where $\mathbf{1}_{N_s}$ is a vector of ones and $\mu\in [0,1]$ is the so-called damping factor \cite{R7}. If $\mu>0$, the Markov chain is ergodic and the algorithm is guaranteed to converge to a unique vector $p$ with positive entries for any initial probability vector $p^0$ \cite{R7}.

\subsubsection{Sentence selection}
Relevance scores described in preceding section rank sentences in terms of relevance to the user-defined query and centrality in the corpus. These scores are further used to select sentences to be included in the summary while not exceeding the summary capacity. A straightforward approach is to select the sentences with maximal relevance scores whose aggregated length does not exceed the capacity, as suggested in early graph-based algorithms \cite{R7,R17}. However, this naive greedy algorithm might favour long sentences over shorter ones \cite{lin2010}. This is not desirable since a combination of shorter sentences may jointly achieve a higher relevance score. Another approach, referred to as \textit{Maximum Relevance} (MR), is to extract the subset $S$ of sentences maximizing the sum of relevance scores, namely
\begin{equation}\label{defMRproblem}
\underset{S\subseteq V}{\max}\text{  }\underset{s\in S}{\sum}p(s)
\text{, }\text{subject to }\underset{s\in S}{\sum}l(s)\leq L.
\end{equation}
A critical issue encountered with this sentence selection approach is that it assumes that the relevance of a summary equals the sum of the relevance scores of its sentences. However, highly scored sentences might exhibit a certain level of redundancy. Indeed, PageRank-like algorithms tend to produce high scores for nodes that are close to each other \cite{lin2010}. A qualitative explanation is that the stationary probability associated to a node is inversely proportional to its hitting time \cite{li2013}. As neighbours in a graph tend to achieve similar hitting times, their PageRank scores are close to each other. In our sentence-based fuzzy hypergraph, this translates into the fact that sentences sharing a large volume of topics achieve similar scores, which implies a certain level of redundancy in the summary.

To alleviate this redundancy issue, previous graph-based summarization algorithms selected sentences based on a \textit{Greedy Redundancy Removal} algorithm (GRR) \cite{xiong2016}. This greedy algorithm selects sentences to be included in the summary $S$ in decreasing order of scores provided that the similarity of each newly selected sentence with sentences already in $S$ does not exceed a predefined threshold. However, a shortcoming of this method is its failure to extract a set of sentences with maximum total relevance. Moreover, while it reduces the level of redundancy in the final summary, there is no guarantee that it properly covers all important topics of the corpus as can be seen from the following example.
\begin{example}\label{example1}
\normalfont
The five sentences below were extracted from a corpus of ten news articles related to the solar eclipse that occurred in U.S. on August 21, 2017\footnote{References to all articles of the corpus are provided in the supplemental materials.}.
\begin{enumerate}
\itemsep0em
\item \textit{"A total eclipse happens when the moon completely covers the sun."} \cite{eclipse4}
\item \textit{"A total eclipse of the sun happens when the moon completely blocks the visible solar disk, casting a shadow on Earth."} \cite{eclipse2}
\item \textit{"The eclipse will cross the U.S. from coast to coast, with totality visible from several major cities and other locations that are easily accessible to millions of people."} \cite{eclipse2}
\item \textit{"The main event will be visible from a relatively narrow path, starting in Oregon and ending in South Carolina."} \cite{eclipse2}
\item \textit{"Swathes of Europe will be able to enjoy a partial eclipse just before sunset."} \cite{eclipse5}
\end{enumerate}
\end{example}
Given the query "How and in what location will the total solar Eclipse occur?", the approximate relevance scores achieved are $2\times 10^{-2}$, $10^{-2}$, $3\times 10^{-3}$, $10^{-3}$ and $5\times 10^{-4}$, respectively. For a summary capacity of $45$ words, according to MR approach, the first two sentences should be selected. GRR method selects sentences $1$ and $3$ which are less redundant. However, sentences $1$, $4$ and $5$ constitute a more informative summary since it better covers the information present in the corpus related to the location from which the eclipse is visible. This example shows that the issue encountered when including redundant sentences in a summary is not the redundancy itself, but rather the fact that redundant sentences may jointly cover a lower amount of information than dissimilar sentences. With that new perspective in mind, we provide a definition of \textit{Topical Coverage} of a set of sentences based on our sentence fuzzy hypergraph. Qualitatively, our goal is to ensure that each sentence in the corpus is semantically similar to sentences in the summary or, in other words, that each sentence in the corpus shares a sufficient number of topics with the sentences in the summary. In probabilistic terms, we define the \textit{semantic relatedness} of a sentence $s$ to a set $S$ of sentences as the probability that a random walker starting in $s$ reaches $S$ in at most one step, with transition probabilities defined by equations \ref{sentenceTransition}. The Topical Coverage of a summary is the sum of the semantic relatedness to the summary of all sentences in the corpus. 
 
\begin{definition}[Topical Coverage]
Given a fuzzy hypergraph $G=(V,E,\psi,w)$, the Topical Coverage of a subset $S\subseteq V$ over $G$ is defined as
\begin{equation}\label{eqnCoverageX}
C(S)=|S|+\underset{\substack{j\in S\\ i\notin S}}{\sum}\underset{e}{\sum}\psi_{je}\frac{\psi_{ie}w(e)}{\underset{f}{\sum}\psi_{if}w(f)}.
\end{equation}
\end{definition}

As we mentioned, for each vertex $i$, the Topical Coverage of $S$ measures the semantic relatedness of $i$ to $S$, namely the probability that a random walker starting in $i$ can reach the set $S$ in no more than one step:
\begin{equation}
p(S|i)=\left\lbrace\begin{array}{ll}
\underset{j\in S}{\sum}\underset{e}{\sum}\psi_{je}\frac{\psi_{ie}w(e)}{\underset{f}{\sum}\psi_{if}w(f)} & \text{ if }i\notin S\\
1 & \text{ if }i\in S
\end{array}\right.
\end{equation}
and $C(S)$ can be rewritten as
\begin{equation}
C(S)=\underset{i\in V}{\sum}p(S|i).
\end{equation}
Hence, maximizing the Topical Coverage ensures that each sentence in the corpus is sufficiently similar to sentences in the summary. The corresponding decision problem can be viewed as a generalization of dominating set problem in the case of fuzzy hypergraphs \cite{garey2002}. We may give another interpretation of topical coverage. When maximizing $C(S)$, the first term in equation \ref{eqnCoverageX} encourages to select short sentences which balances the fact that long sentences tend to have higher relevance scores. The second term of $C(S)$ can be written as
\begin{equation}
\underset{e}{\sum}\underset{\substack{j\in S\\i \notin S}}{\sum}p(j|e)p(e|i)
\end{equation}
which encourages hyperedges to have a balanced number of incident vertices respectively in $S$ and in $V\setminus S$. This implies that each topic is indeed covered by sentences in $S$ while reducing the risk of including semantically redundant sentences covering the exact same topics. For this reason, we refer to $C(S)$ as the \textit{Topical Coverage} of $S$.

Combining both criteria of \textit{Relevance} and \textit{Topical Coverage}, our proposed method seeks sentences that are individually relevant and that jointly cover the semantic content of the corpus. This translates into a multi-objective discrete optimization problem.

\begin{definition}[Maximum Relevance and Coverage Problem (MRC)]\label{defMRC}
Given a set $V$ of sentences extracted from a corpus, a summary capacity $L$ and a set of relevance scores $\{p(s):s\in V\}$, the Maximum Relevance and Coverage Problem is
\begin{equation}\label{defObjFct}
\underset{S\subseteq V}{\max}\text{ }(1-\nu)\underset{s\in S}{\sum}p(s)+\frac{\nu}{N_s}C(S)\text{\normalfont, subject to }\underset{s\in S}{\sum}l(s)\leq L
\end{equation}
where $\{l(s):s\in V\}$ are the sentence lengths, $\nu\in [0,1]$ and $N_s=|V|$.
\end{definition}

The following theorem shows that MRC problem is NP-hard.

\begin{theorem}
For a set $V$ of sentences, a capacity $L$ and relevance scores $\{p(s):s\in V\}$, the decision problem associated to MRC is NP-hard.
\end{theorem}
\begin{proof}
In the particular case of $\nu=0$, MRC is equivalent to $0-1$ Knapsack problem in which $V$ is the set of items, relevance scores $\{p(s):s\in S\}$ are the item values and sentence lengths $\{l(s):s\in V\}$ are the item weights. \end{proof}

As MRC problem is NP-hard, we provide a polynomial time algorithm providing an approximate solution to it with a constant approximation factor. Various scalable algorithms for finding near optimal solutions to NP-hard problems build on the submodularity and non-decreasing property of the associated objective function. These properties are defined below (definition \ref{defSubmNonDecr}).
\begin{definition}\label{defSubmNonDecr}
Given a finite set $V$, a function $F:P(V)\rightarrow\mathbb{R}$ (where $P(V)$ denotes the power set of $V$) is submodular if $\forall S\subseteq T\subset V$ and $r\in V\setminus T$
\begin{equation}
F(S\cup\{r\})-F(S)\geq F(T\cup\{r\})-F(T)
\end{equation}
and it is monotonically non-decreasing if $\forall S\subset V$ and $r\in V\setminus S$
\begin{equation}
F(S\cup\{r\})\geq F(S).
\end{equation}
\end{definition}
Our approximation algorithm builds on the property that the objective function of MRC problem is submodular and monotonically non-decreasing, which is proved in theorem \ref{theoremSubm}.
\begin{theorem}\label{theoremSubm}
The objective function $F:P(V)\rightarrow [0,1]$ of Maximum Relevance and Coverage problem (equation \ref{defObjFct}) is submodular and monotonically non-decreasing.
\end{theorem}
\begin{proof}
Let $V$ be the set of sentences in the corpus, $S\subseteq V$ be the selected sentences for the summary and
\begin{equation}
R(S)=\underset{s\in S}{\sum}p(s).
\end{equation}
Then $F$ becomes
\begin{equation}
F(S)=(1-\nu)R(S)+\frac{\nu}{N_s} C(S).
\end{equation}
Also let
\begin{equation}
p(j|i)=\underset{e}{\sum}\psi_{je}\frac{\psi_{ie}w(e)}{\underset{f}{ \sum}\psi_{if}w(f)}.
\end{equation}

Defining $F(\emptyset)=0$, we have $\forall S\subset V$ and $\forall r\in V\setminus S$ 

\begin{equation}
\begin{array}{rcl}
N_sF(S\cup \{r\})&=&(1-\nu)N_sR(S\cup \{r\})+\nu C(S\cup \{r\})\\
&\geq& (1-\nu)N_sR(S) + \nu(|S| + \underset{j\in S}{\sum}p(j|r)\\
&&+ \underset{\substack{j\in S\\i\notin S\cup\{r\}}}{\sum}p(j|i)+ \underset{i\notin S\cup\{r\}}{\sum}p(r|i))\\
&\geq& (1-\nu)N_sR(S) + \nu(|S| + \underset{\substack{j\in S\\i\notin S}}{\sum}p(j|i))=N_sF(S)
\end{array}
\end{equation}
which proves that $F$ is monotonically non-decreasing. To prove $F$ is submodular, we observe that $\forall S\subseteq T\subset V$ and $r\in V\setminus T$
\begin{equation}\label{eqnModular}
\begin{array}{l}
N_s((F(S\cup\{r\})-F(S))-(F(T\cup\{r\})-F(T)))\\
=\nu( \underset{i\notin S\cup\{r\}}{\sum} p(r|i)-\underset{i\notin T\cup\{r\}}{\sum} p(r|i))+\nu( \underset{j\in T}{\sum}p(j|r)-\underset{j\in S}{\sum}p(j|r)).
\end{array}
\end{equation}
Considering the first term in equation \ref{eqnModular}, we have
\begin{equation}
\underset{i\notin S\cup\{r\}}{\sum} p(r|i)-\underset{i\notin T\cup\{r\}}{\sum} p(r|i)=\underset{i\in T\setminus S}{\sum}p(r|i)\geq 0
\end{equation}
and for the second term, we have
\begin{equation}
\underset{j\in T}{\sum}p(j|r)-\underset{j\in S}{\sum}p(j|r)=\underset{j\in T\setminus S}{\sum} p(j|r)\geq 0
\end{equation}
which completes the proof of submodularity.
\end{proof}

\begin{algorithm}[H]~\\
INPUT: Set $V$ of sentences, parameter $\nu$, capacity $L$, sentence lengths $\{l_s:1\leq s\leq N_s\}$, Hypergraph $H(\{1,...,N_s\},\{1,...,K\},\psi,w)$.\\
OUTPUT: Set $S$ of indices of sentences to be included in the summary.\\
\textbf{for each} $j,i\in\{1,...,N_s\}$: compute $p(j|i)$ and $p^q(j|i)$ (equations \ref{sentenceTransition}-\ref{queryBiasedTransition})\\
Compute sentence relevance scores $\{p_i:1\leq i\leq N_s\}$ (equation \ref{pagerank})\\
Let $Z\leftarrow V$, $T\leftarrow \emptyset$, $\rho\leftarrow 0$\\
\textbf{for each} $j\in\{1,...,N_s\}$: $\pi_j\leftarrow \frac{1}{l_j}(\frac{\nu}{N_s}(1+\underset{i\neq j}{\sum} p(j|i))+(1-\nu)p_j)$\\
\textbf{while} $Z\neq\emptyset$ \textbf{and} $\rho\leq L$:\\
	\Indp $s^*\leftarrow \underset{s\in Z}{\text{argmax}}\text{ }\pi_s$\\
	$Z\leftarrow Z\setminus \{s^*\}$, $T\leftarrow T\cup\{s^*\}$, $\rho\leftarrow \rho+l_{s^*}$\\
	\textbf{for each} $j\in\{1,...,N_s\}$: $\pi_j \leftarrow \pi_j-\frac{\nu}{N_sl_j}(p(s^*|j)+p(j|s^*))$.\\
\Indm Let $Q\leftarrow\{\{s\}\text{: }l(s)\leq L,s\in V\}$\\
Let $S\leftarrow\underset{S\in\{T\}\cup Q}{\text{argmax}}(1-\nu)\underset{s\in S}{\sum}p(s)+\frac{\nu}{N_s}(|S|+\underset{j\in S, i\notin S}{\sum}p(j|i))$
\caption{Maximal Relevance and Coverage (MRC) Algorithm}
\label{algorithmSUBM}
\end{algorithm}
\FloatBarrier

MRC problem consists in the maximization of a submodular and non-decrea-sing function under a \textit{capacity constraint}. We formulate polynomial time approximation algorithm \ref{algorithmSUBM} for solving MRC problem. Our method builds on an approach proposed by Lin et al. \cite{lin2010} for the maximization of monotonically non-decreasing submodular functions under a budget constraint. We prove in theorem \ref{guaranteeTheorem} that algorithm \ref{algorithmSUBM} provides a near-optimal solution to MRC problem with a relative performance guarantee. The proof relies on the submodularity and non-decreasing property proved in theorem \ref{theoremSubm}. The time complexity of algorithm \ref{algorithmSUBM} is dominated by the computation of relevance scores and the sentence selection step which have a time complexity of $O(\tau N_s^2)$ where $\tau$ is the number of iterations for the iterative computation of relevance scores. The final summary is produced by aggregating the sentences selected by algorithm \ref{algorithmSUBM}.

\begin{theorem}\label{guaranteeTheorem}
Let $F$ be the objective function of MRC problem, then algorithm \ref{algorithmSUBM} produces a summary $S$ verifying
\begin{equation}
F(S)\geq (1-e^{-\frac{1}{2}})F(S^*)
\end{equation} 
where $S^*$ is the optimal solution of MRC problem.
\end{theorem}
\begin{proof}
The objective function $F$ of MRC problem (definition \ref{defMRC}) is submodular and monotonically non-decreasing from theorem \ref{theoremSubm}. Hence,
\begin{equation}
\underset{S\subseteq V}{\max}\text{ }F(S)\text{, subject to }\underset{s\in S}{\sum}l(s)\leq L
\end{equation}
corresponds to the maximization of a submodular and monotonically non-decreasing function under a budget constraint \cite{lin2010}. Let $T$ be the set of sentences obtained by iteratively appending each sentence $r$ of the corpus to $T$ maximizing
\begin{equation}
\frac{F(T\cup \{r\})-F(T)}{l(r)}
\end{equation}
provided that $l(r)+\underset{s\in T}{\sum}l(s)\leq L$. Also let $Q$ be the set of sentences that are individually satisfying the capacity constraint, namely $Q=\{\{s\}\text{: }l(s)\leq L,s\in V\}$. Let the final summary consist of the set $S^F$ of sentences satisfying
\begin{equation}
S^F=\underset{S\in\{T\}\cup Q}{\text{argmax}}F(S).
\end{equation}
Then, from theorem 1 in \cite{lin2010}, the summary $S^F$ is a near optimal solution to MRC problem satisfying
\begin{equation}
F(S^F)\geq (1-e^{-\frac{1}{2}})F(S^*)
\end{equation} 
where $S^*$ is the optimal solution of MRC problem. Moreover, the set $S^F$ of sentences corresponds to the summary produced by algorithm \ref{algorithmSUBM}. 
\end{proof}

\section{Experiments and evaluation}\label{experimentSection}
We present experimental results obtained by testing our summarization framework on real-world datasets. We conduct four sets of experiments: a qualitative analysis of a summary produced by our MRC algorithm, a parameter tuning, an assessment of the relevance of each step of our method and a comparison with state-of-the-art summarizers.

For the first experiment, we gathered a new dataset of recent newspaper articles. For the other experiments, we make use of the benchmark datasets of \textit{Document Understanding Conferences} DUC05, DUC06 and DUC07 for query-oriented text summarization \cite{duc05,duc06,Dang2007}. Each data sample consists of a corpus of news articles related to a specific topic, a query and a set of query-oriented \textit{reference summaries} written by humans. The datasets contain 50, 50 and 45 different corpora. Each corpus consists of about 30 news articles of 1000 words on average. The length of the reference summaries is restricted to 250 words, so we set the summary capacity parameter $L$ to $250$.

\subsection{Example of summary}
As a preliminary experiment, we show an example of summary produced by our system. Benchmark datasets for summarization usually consist of corpora of about twenty to fifty papers of about a thousand words each. Hence, we gathered a corpus of $20$ newspaper articles of $715$ words on average (for a total of $15015$ words) related to the migration crisis faced by Europe in recent years\footnote{References to all articles of the corpus are provided in the supplemental materials.}.

A summary is generated for the following query: "\textit{Describe the challenges faced by the European Union related to migration from Subsaharian Africa and the Middle East. What policies are implemented by the members of the European Union to address these challenges?}". Table \ref{tabNewSummary} displays the top eight sentences returned by our algorithm, along with some of the corresponding topics. These topics correspond to topics inferred by our topic model that we labelled with explicit names such as "migration", "EU" or "challenges". We make the following observations regarding the summary.

\begin{table}[h]
\begin{center}
\resizebox{\textwidth}{!}{\fontsize{5}{4}\selectfont
\begin{tabular}{|L{0.2cm}|L{7cm}|L{1.5cm}|}
\hline
\rule{0pt}{3ex} & \textbf{Sentences selected for our summary} & \textbf{Topics covered by each sentence}\\
\hline
\rule{0pt}{3ex} \cellcolor{gray!25} 1. & "A record 1.3 million migrants applied for asylum in the 28 member states of the European Union, Norway and Switzerland in 2015 - nearly double the previous high water mark of roughly 700,000 that was set in 1992 after the fall of the Iron Curtain and the collapse of the Soviet Union, according to a Pew Research Center analysis of data from Eurostat, the European Union's statistical agency." \cite{summary1} & Countries of destination, EU, Asylum, Migration\\
\hline
\rule{0pt}{3ex} \cellcolor{gray!25} 2. & "Some view this as a humanitarian crisis and others see it as a challenge and a threat." \cite{summary2}& Challenges, Humanitarian\\
\hline
\rule{0pt}{3ex} \cellcolor{gray!25} 3. & "Security, political, and social concerns compound these challenges." \cite{summary3} & Challenges, Political\\
\hline
\rule{0pt}{3ex} \cellcolor{gray!25} 4. & "The study commissioned by UNHCR found that the profiles and nationalities of people arriving in Libya have been evolving over the past few years, with a marked decrease in those originating in East Africa and an increase in those from West Africa, who now represent well over half of all arrivals to Europe through the Central Mediterranean route from Libya to Italy (over 100,000 arrivals in 2016)." \cite{summary4} & Countries of origin, Countries of destination, Migration\\
\hline
\rule{0pt}{3ex} \cellcolor{gray!25} 5. & "The dislocation of large parts of the population in Syria and other conflict zones is, first and foremost, a humanitarian catastrophe with important ramifications across many countries in the Middle East, Europe, and beyond." \cite{summary3} & Conflict, Humanitarian, Countries of origin\\
\hline
\rule{0pt}{3ex} \cellcolor{gray!25} 6. & "Border restrictions in the Western Balkans and a deal with Turkey led to a significant decline in arrivals by sea to Greece of asylum seekers and other migrants, while boat migration from North Africa to Italy remains steady." \cite{summary5} & Policy, Countries of origin, Countries of destination, Migration, EU\\
\hline
\rule{0pt}{3ex} \cellcolor{gray!25} 7. & "Furthermore, authors warn that tensions between immigrants and native workers, fueled by an unsubstantiated but widespread belief that immigrants "undercut" natives in the labor market, may lead to immigrant-backlash and hinder the social and economic integration of immigrants, especially in countries where immigration-related conflicts are already evident." \cite{summary6} & Challenges, Social, Labor, Conflicts\\
\hline
\rule{0pt}{3ex} \cellcolor{gray!25} 8. & "In particular, Europe faces a major demographic challenge: our population is aging, and, in many countries, shrinking." \cite{summary2} & Challenges, Demography, EU\\
\hline
\end{tabular}}
\end{center}
\caption{Example of summary from a corpus of $20$ articles related to migration crisis in Europe.}
\label{tabNewSummary}
\end{table}
\FloatBarrier

First, several different topic labels are assigned to each sentence, which captures the multiplicity of topics covered by sentences. In contrast, previous hyper-graph-based summarization algorithms were based on the classification of each sentence in a single cluster \cite{wanng2013,xiong2016}. 

Second, we observe that selected sentences exhibit a certain level of lexical redundancy since various words appear several times in the sentences (e.g. the word "Europe"). This is due to the fact that our Relevance and Topical Coverage criterion ensures that the resulting summary presents a good coverage of our fuzzy hypergraph without further restriction on the level of lexical redundancy.

Third, we observe that selected sentences do not necessarily have terms in common with the query (such as sentences 2, 3 and 7 which share only one non-trivial term with the query). This highlights the ability of our method to measure a query similarity based on common topics rather than common terms as done in the majority of existing query-focused summarization methods. 

Finally, we observe that our summary covers the main themes present in the corpus and already described above:

\begin{itemize}
\itemsep0em
\item sentence 1: primary European countries that were impacted by the crisis,
\item sentences 2 and 3: challenges faced by European countries,
\item sentence 4: nationalities of migrants, primary European countries that were impacted by the crisis,
\item sentence 5: causes of the migration outbreak,
\item sentence 6: policies implemented by European countries,
\item sentence 7: (socioeconomic) challenges faced by European countries,
\item sentence 8: (demographic) challenges faced by European countries.
\end{itemize}

\subsection{Metrics for summary evaluation}\label{metricSection}
Two aspects of our automatically generated summaries are evaluated, namely their \textit{content} and \textit{diversity}. These aspects are evaluated based on a comparison with reference summaries written by humans. The content evaluation verifies whether the information coverage of a summary matches that of the reference summaries. The diversity test checks whether the candidate summary presents sufficient diversity in its content and little redundancy. For content evaluation, we make use of ROUGE toolkit \cite{lin2004} which includes several popular recall-based metrics for summary evaluation. Each metric measures the overlap in different types of word sequences between reference summaries and a candidate summary. We make use of ROUGE-N which measures the number of N-grams that are found in both the set of reference summaries and the candidate summary divided by the total number of N-grams in the reference summaries. In particular, as suggested in \cite{wanng2013}, we use ROUGE-2 metric to evaluate the content of our candidate summaries. We also use ROUGE-SU4 metric which counts both the number of common unigrams (terms) and 4-skip-bigrams, namely pairs of words that are separated by at most four words in a summary. ROUGE-SU4 allows for more flexibility in word ordering than ROUGE-N. Hence, we use ROUGE-SU4 as a reference metric and we report ROUGE-2 for the sake of completeness. The parameter setting of ROUGE metrics is done according to DUC evaluations: jackknife resampling is performed, words in summaries are stemmed but stop-words are not removed. More information can be found in the description of ROUGE toolkit \cite{lin2004} and in the description of DUC evaluations \cite{duc05,duc06,Dang2007}.

Finally, to evaluate the diversity of a summary, we measure the \textit{Normalized Entropy} of its term distribution $[p_1,...,p_{N_t}]$, namely
\begin{equation}
H(p)=-\frac{1}{\log N_t}\sum_{i=1}^{N_t}p_i\log p_i.
\end{equation}

The normalized entropy is $0$ for a sentence containing a single term and it is $1$ for a uniform distribution over terms. Hence, it can be interpreted as a measure of the \textit{Lexical Diversity} of a summary. It gives an indication of the non-redundancy of the information present in it.

\subsection{Parameter tuning}\label{sectionParamTuning}

For our HDP-based model, the implementation of \cite{wang2012} is used which is based on Gibbs sampling and can be adapted for multiple-level HDP. The values of parameters $\lambda$, $\mu$ and $\nu$ are set to values of $0.9$, $0.99$ and $0.2$ respectively and the values of the four concentration parameters are tuned. A validation set consisting of $90\%$ of corpora of DUC07 dataset is randomly selected and, for each corpus and for different values of the concentration parameters, the model is evaluated via a leave-one-out cross-validation due to the limited size of the corpora. We use a method similar to that of \cite{xiong2016} for parameter tuning, with values of $\gamma$ in the range $1,...,10$, of $\beta$ from $0.5$ to $5$ and of $\alpha$ from $0.25$ to $2.5$. Highest ROUGE-SU4 scores are achieved for values of $7.0$ for $\gamma$, $1.5$ for $\beta$ and $0.75$ for $\alpha$. We choose smaller values for $\alpha$ than for $\beta$ since we expect the level of variability of topics within sentences to be smaller than that observed at a document-level. The same observation is valid when comparing $\beta$ (documents) to $\gamma$ (corpus). Finally we choose the value of concentration parameter $\zeta$ of the symmetric Dirichlet prior to be $0.5$ in accordance with what was suggested in the original version of HDP \cite{hdp}. 

We now conduct an experiment to find suitable values of the main parameters of our method, namely the query balance $\lambda$, the damping factor $\mu$ and the coverage balance $\nu$. We apply an alternating maximization strategy in which two parameters are set to a value in $[0,1]$ and we seek the value of the third parameter that maximizes ROUGE-SU4. The optimal values we obtain for the three parameters using cross-validation are approximately $\lambda=0.75$, $\mu=0.99$ and $\nu=0.35$. A value of $\lambda=0.75$ gives more weight to the score propagation term than to the query relevance, $\mu=0.99$ is a standard value for the damping factor of a PageRank-like algorithm \cite{R17} and $\nu=0.35$ gives more weight to the Relevance criterion than the Topical Coverage criterion. Next we show the variation of both ROUGE-SU4 and Lexical Diversity with the value of each parameter. In each case we set two parameters to the values above and we let the third parameter vary between $0$ and $1$. We computed the average ROUGE-SU4 and Lexical Diversity scores achieved by each candidate summary produced for each corpus of DUC07 dataset.

We first set the values of $\mu$ and $\nu$ respectively to $0.99$ and $0.35$ and we let $\lambda$ vary between $0$ and $1$. Figure \ref{lambda_graphs} displays the evolution of ROUGE-SU4 and Lexical Diversity as a function of $\lambda$. We observe that ROUGE-SU4 reaches a peak close to $\lambda=0.75$. We recall that parameter $\lambda$ commands the extent to which scores are learnt from the query relevance or from the propagation of scores across the fuzzy hypergraph. $\lambda=0$ gives credit to the query relevance only while $\lambda=1$ focuses on propagation. Our experiment shows that the propagation accross the fuzzy hypergraph improves the quality of the output over that obtained with query relevance only, with a sharp initial increase in quality. A maximum ROUGE-SU4 score of $0.1792$ is achieved for $\lambda=0.75$. However, the score varies smoothly above $0.17$ when $\lambda$ lies in the interval $[0.2,0.8]$. This shows that our method is not highly sensitive to the value of $\lambda$. In figure \ref{lambdaDiversity}, we display the evolution of the Lexical Diversity with $\lambda$. We observe that the lexical diversity does not vary significantly for $\lambda\in [0,0.8]$ and it subsequently increases with $\lambda$ as low values of $\lambda$ emphasize on the query relevance while high values of $\lambda$ give more weight to the score propagation term which results in lexically diverse summaries.

Next, we set the values of $\lambda$ to $0.75$ and $\nu$ to $0.35$ and we let $\mu$ vary between $0$ and $0.99$. The damping factor is a parameter that ensures the convergence of our PageRank-like algorithm by letting the random walker jump to any node of the hypergraph with a small probability $(1-\mu)$ at each step. Figure \ref{muRougeSU} shows that ROUGE-SU4 reaches a peak for a value close to $0.99$. The Lexical Diversity of the summary displayed in graph \ref{muDiversity} obviously rises when $\mu$ decreases but this is due to the fact that a lower value of $\mu$ results in similar scores for all sentences.

Finally, we set the values of $\lambda$ and $\mu$ respectively to $0.75$ and $0.99$ and we let coverage balance parameter $\nu$ vary between $0$ and $1$ (figure \ref{nuRougeSU}). We recall that parameter $\nu$ determines the balance between Relevance and Topical Coverage criteria in the sentence selection process. $\nu=0$ focuses on the Relevance criterion while $\nu=1$ focuses on the Topical Coverage criterion. We observe that ROUGE-SU4 reaches a peak around $\nu=0.35$. The impact of the Topical Coverage criterion is significant since $\nu=0.35$ greatly increases ROUGE-SU4 score over $\nu=0$. Moreover, any value of $\nu$ in the interval $[0.1,0.7]$ results in a score above $0.17$ which confirms the low sensitivity of our method to the value of parameter $\nu$. On the other hand we observe in figure \ref{nuDiversity} that the Lexical Diversity of the summary grows with $\nu$ which shows that, while our Topical Coverage criterion is meant to increase the topical diversity of the summary, it also reduces the lexical redundancies compared to a selection based on relevance only.

\begin{figure}[!h]
\centering
\subfigure{
    \includegraphics[width=.45\textwidth]{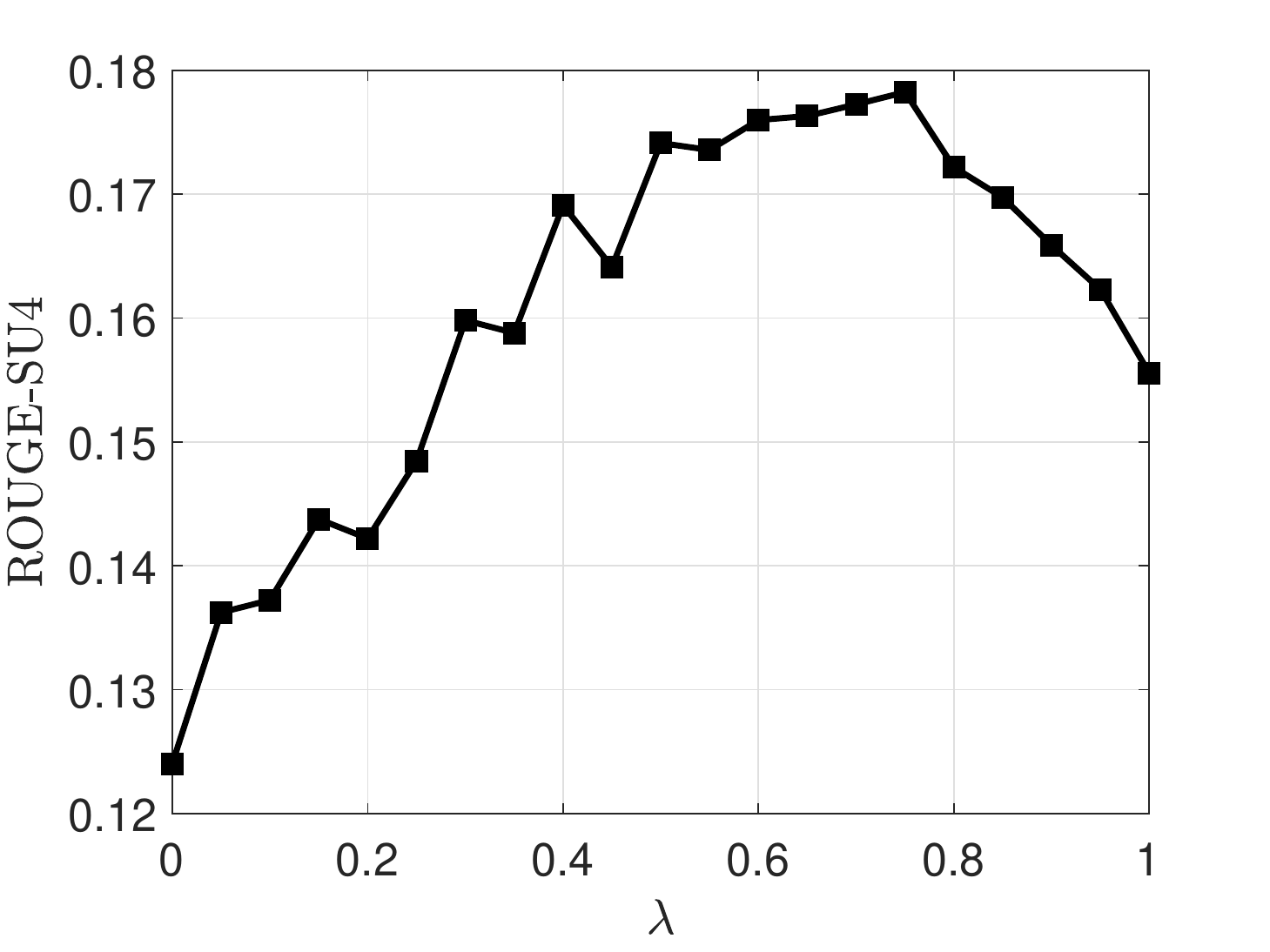}
    \label{lambdaRougeSU}
}
\subfigure{
	\includegraphics[width=.45\textwidth]{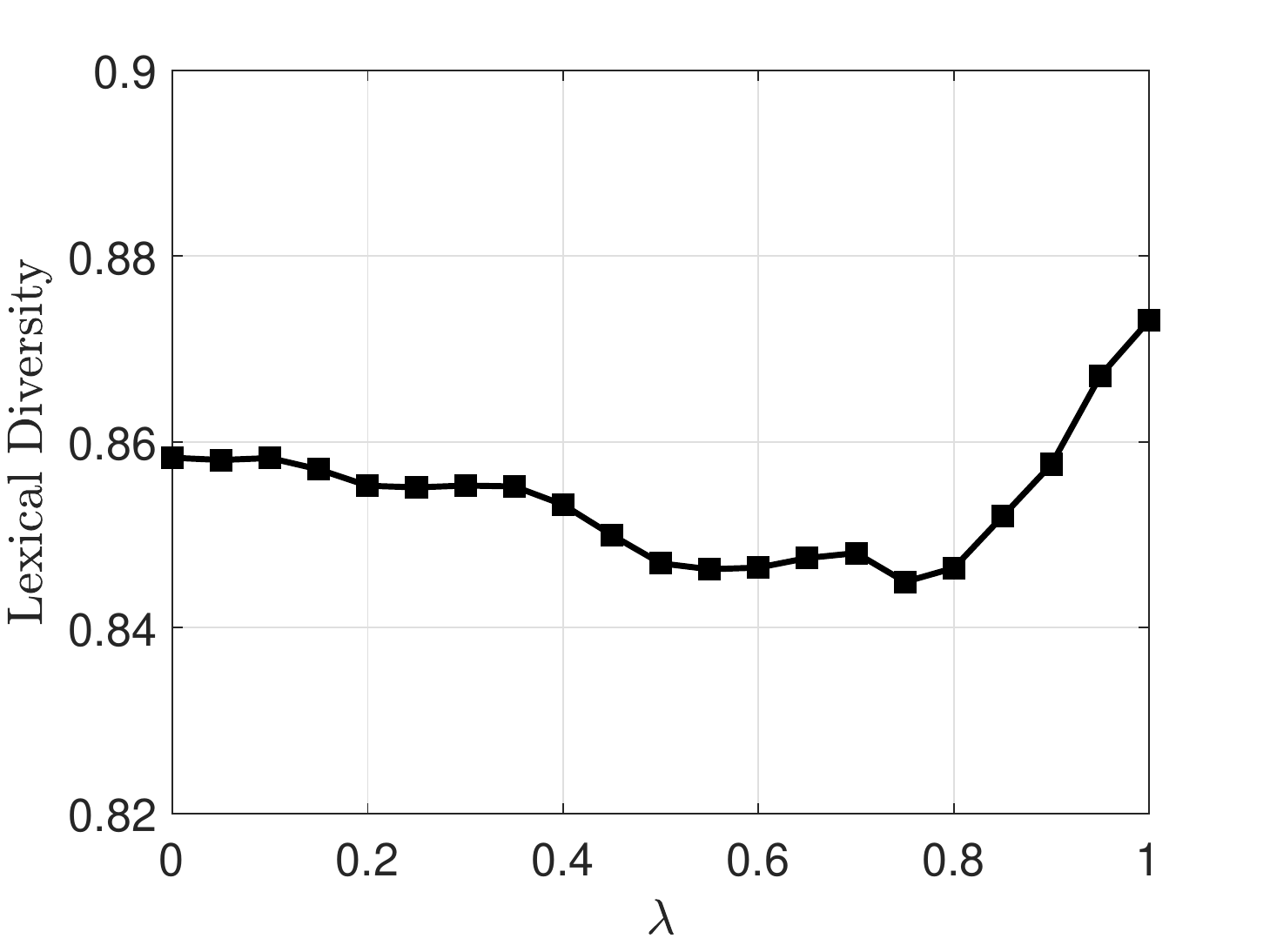}
	\label{lambdaDiversity}
}
\caption{ROUGE-SU4 and Lexical Diversity as a function of $\lambda$ for $\mu=0.99$ and $\nu=0.35$.}
\label{lambda_graphs}
\end{figure}

\begin{figure}[!h]
\centering
\subfigure{
    \includegraphics[width=.45\textwidth]{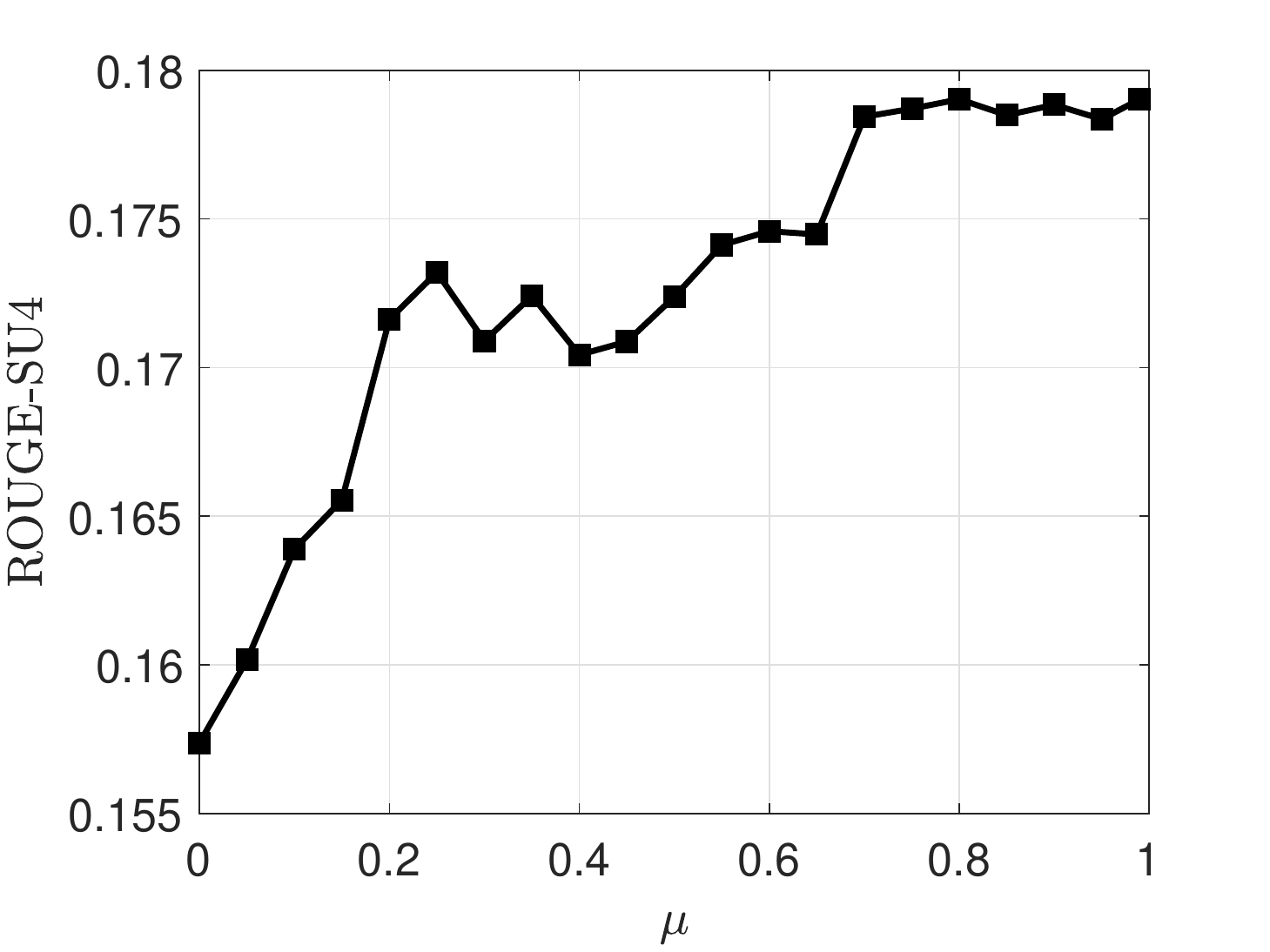}
    \label{muRougeSU}
}
\subfigure{
	\includegraphics[width=.45\textwidth]{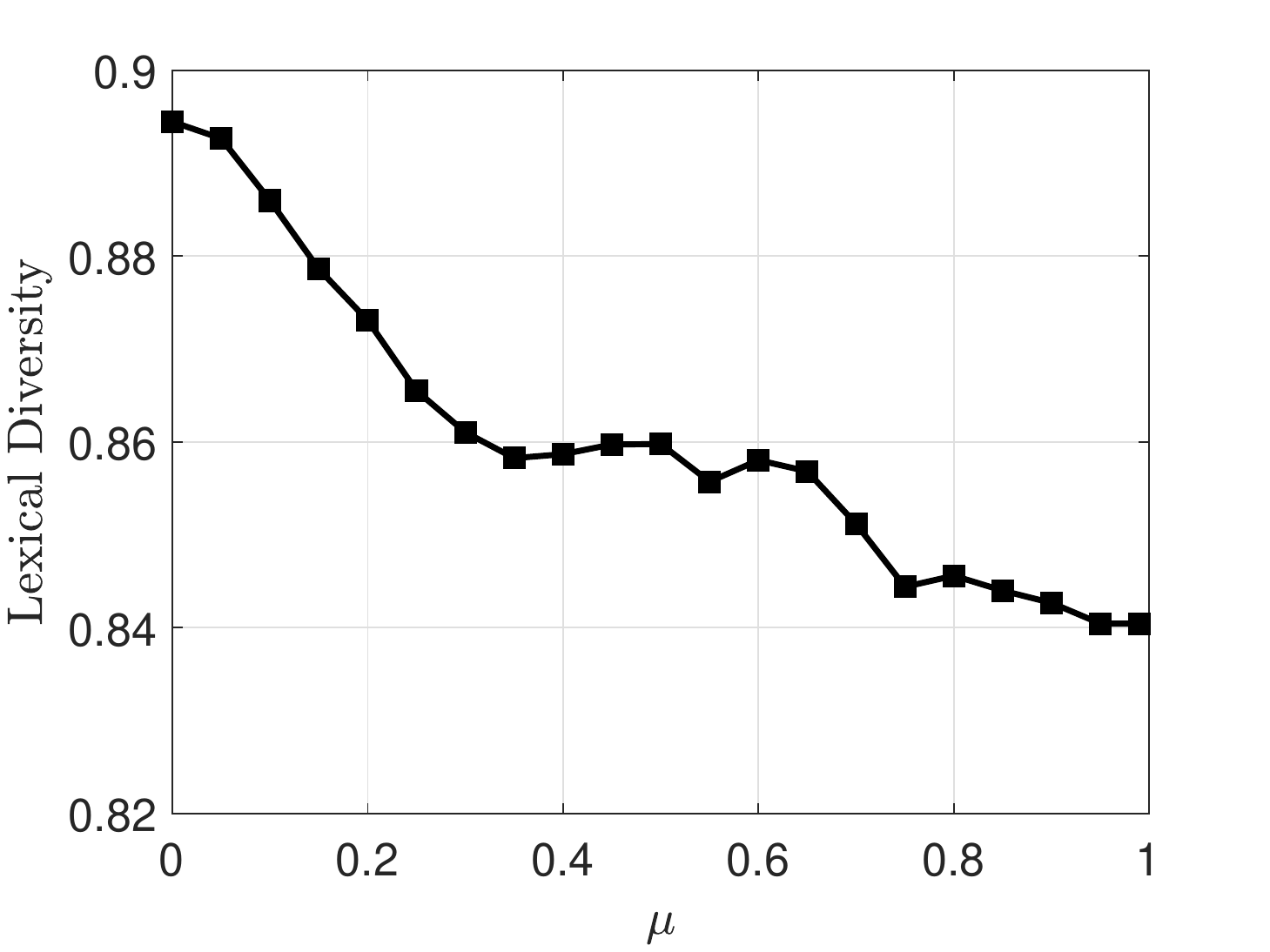}
	\label{muDiversity}
}
\caption{ROUGE-SU4 and Lexical Diversity as a function of $\mu$ for $\lambda=0.75$ and $\nu=0.35$.}
\label{mu_graphs}
\end{figure}

\begin{figure}[!h]
\centering
\subfigure{
    \includegraphics[width=.45\textwidth]{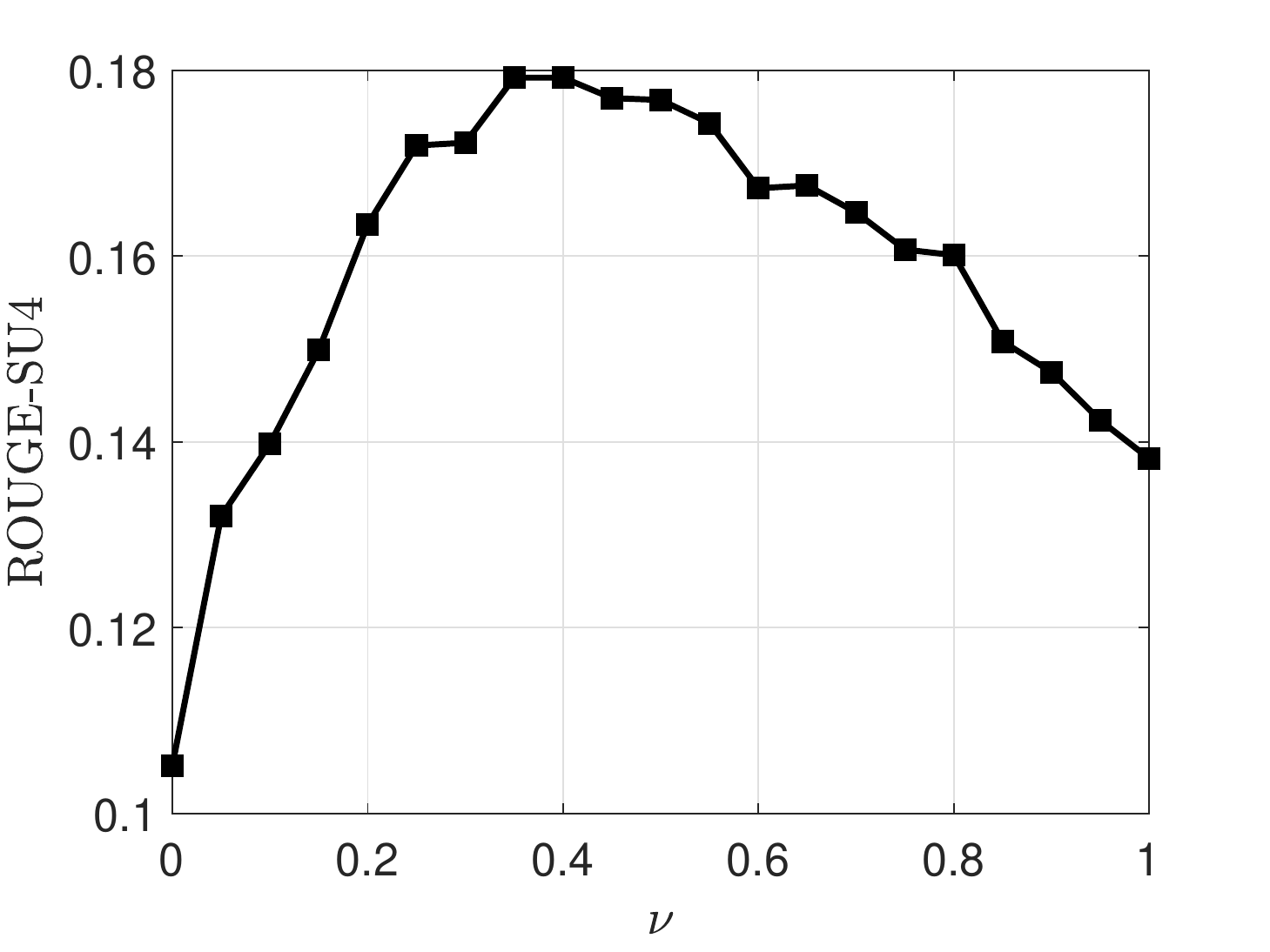}
    \label{nuRougeSU}
}
\subfigure{
	\includegraphics[width=.45\textwidth]{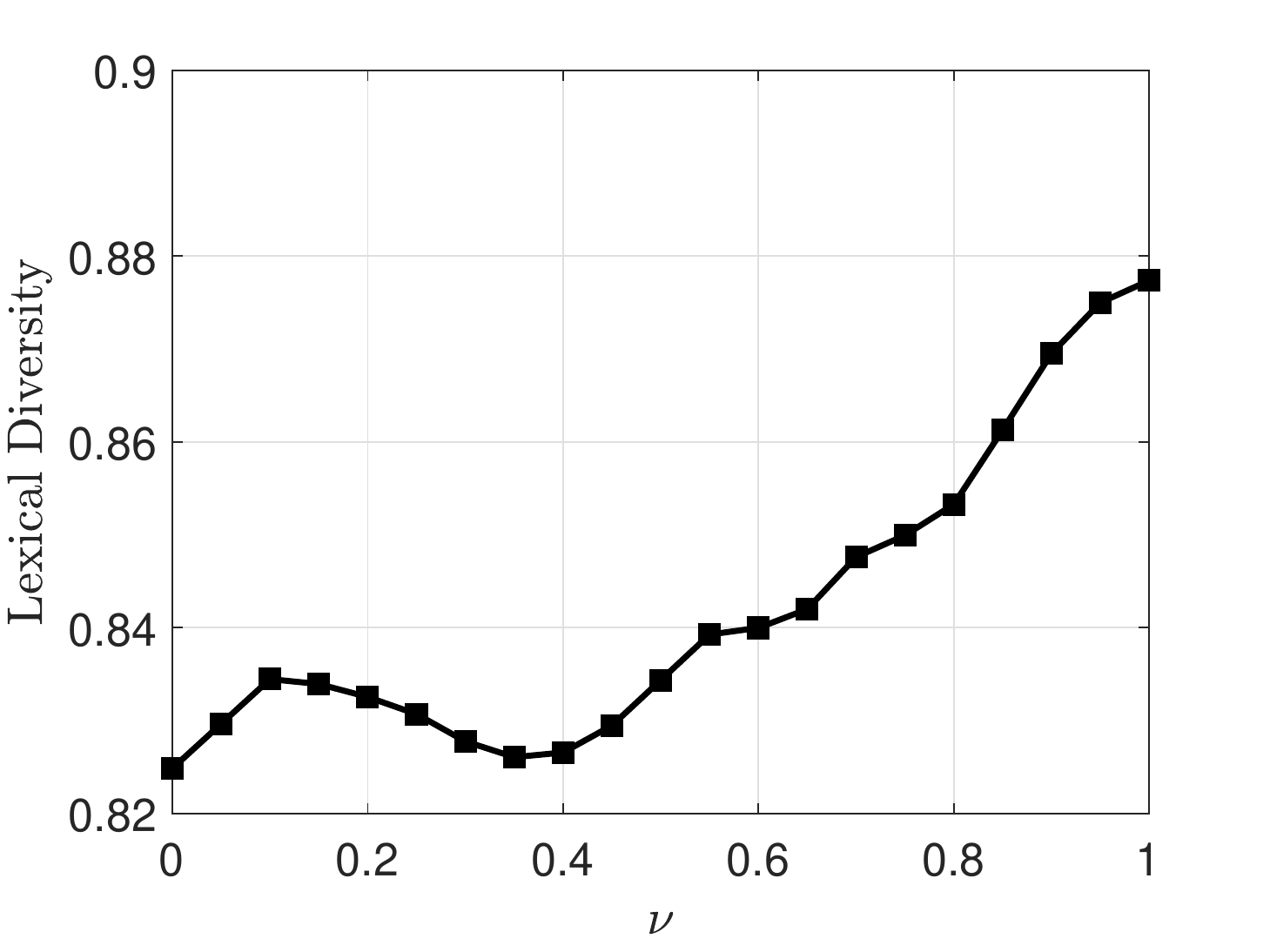}
	\label{nuDiversity}
}
\caption{ROUGE-SU4 and Lexical Diversity as a function of $\nu$ for $\lambda=0.75$ and $\mu=0.99$.}
\label{nu_graphs}
\end{figure}

\subsection{Testing the hypergraph construction}
This experiment shows the relevance of our hypergraph construction method. Since other methods were already proposed to incorporate topical or cluster relationships in graph-based summarization frameworks \cite{wanng2013,xiong2016}, we test other models for the hyperedges of our fuzzy hypergraph.

We present five other popular ways to infer relationships between sentences. The first method called \textit{Latent Dirichlet Allocation} (LDA) \cite{blei2003} is a probabilistic topic model which associates a single distribution over a predefined number of topics to each document and represents each topic as a distribution over terms. The main differences between LDA and HDP are first that LDA takes the number of topics as a parameter whose value must be determined by cross validation \cite{hdp}. Second, LDA does not provide a flexible hierarchical framework as HDP does. Hence, sentence topic tags are extracted from document topic tags using a heuristic described in \cite{arora2008}. The second hyperedge model builds on \textit{Terms} instead of higher-level topical relationships. Each term defines a hyperedge connecting the sentences in which the term is present. The term frequency within each sentence defines the hyperedge distribution over sentences. The weight of each hyperedge $t$ is the product of the term frequency $\text{tfc}(t)$ and the isf weight $\text{isf}(t)$ (equation \ref{isfDef}).

The remaining hyperedge models are based on the detection of  clusters of lexically similar sentences. Clusters are obtained by applying clustering algorithms to tfisf representations of sentences \cite{blake2006}. Each sentence cluster represents a hyperedge over sentences and the hyperedge weights are defined as the cosine similarity between the tfisf representation of the corresponding sentence cluster and the tfisf representation of the whole corpus as suggested in \cite{wanng2013,xiong2016}. Three clustering algorithms are tested using the cosine distance between tfisf representations as a distance metric over sentences. The first algorithm is $k$\textit{-means} and, in particular, Lloyd's algorithm \cite{Lloyd1982}. The second method is \textit{agglomerative clustering} \cite{rokach2005}, a popular hierarchical clustering method. Finally, a nonparametric version of \textit{DBSCAN} clustering algorithm \cite{wanng2013} is tested. \cite{wanng2013} showed that DBSCAN best captures groups of lexically similar sentences, due to its ability to remove outliers. As suggested in \cite{wanng2013}, additional pairwise hyperedges based on the cosine similarity between tfisf representations of sentences are also included in the hypergraph.

The values of the parameters of the algorithms are set in the same way as we did for parameters of our MRC algorithm: $k$-means is ran for a number of clusters of $10$ to $150$ with steps of $5$ and the optimal number of clusters is $70$. Similarly, for LDA, the optimal number of topics is $55$. Finally, the stopping criterion of Agglomerative Clustering requires a threshold. Its optimal value is searched in the interval $[0,1]$ and found to be $0.21$.

Table \ref{tabHyperedges} displays ROUGE-2 and ROUGE-SU4 scores and corresponding $95\%$ confidence intervals for all seven hyperedge models, including our MRC algorithm with parameter values given in section \ref{sectionParamTuning}. We do not display the Lexical Diversity measure since diversity of summaries is not enforced by our sentence ranking step. We observe that our MRC algorithm outperforms LDA-based approach by $14\%$ in terms of ROUGE-SU4 which confirms that the hierarchical structure of our topic model provides a more accurate model for the distribution of sentences over topics. Moreover, it also outperforms the term-based model by $5\%$ in terms of ROUGE-SU4 which shows that the extraction of semantically related terms in the form of topics increases the quality of the resulting summary. Finally our MRC algorithm outperforms the cluster-based approaches and, in particular, it outperforms best performing DBSCAN algorithm by $5\%$ in terms of ROUGE-SU4. This justifies our choice of a topic model tagging sentences with multiple topics instead of a cluster-based approach classifying each sentence in a single cluster. Overall our algorithm outperforms other hyperedge models by $25\%$ in terms of ROUGE-2 and by $9\%$ in terms of ROUGE-SU4, on average.

\begin{table}[!h]
\begin{center}
\fontsize{9}{7}\selectfont
\begin{tabular}{|| c | c | c | c ||}
   \hline
   \rule{0pt}{2ex} Hyperedge model & ROUGE-2 & ROUGE-SU4 \\
   \hline
	\rule{0pt}{2ex} \textbf{MRC} & $\mathbf{0.12745 (0.11791 - 0.13699)}$ & $\mathbf{0.1792 (0.17065-0.18775)}$\\
	LDA & $0.09336  (0.081 - 0.10572)$ & $0.15666  (0.15078 - 0.16254)$\\
	TERMS & $0.1131  (0.10833 - 0.11786)$ & $0.1708  (0.16616 - 0.17544)$\\
	KMEANS & $0.10574  (0.09366 - 0.11781)$ & $0.16831  (0.16095 - 0.17567)$\\
	AGGLOMERATIVE & $0.09251  (0.07899 - 0.10603)$ & $ 0.1534 (0.14236 - 0.16444)$\\
	DBSCAN & $0.10636  (0.09475 - 0.11797)$ & $0.17049  (0.16385 - 0.17713)$\\
   \hline	
\end{tabular}
\end{center}
\caption{Performance of our MRC algorithm and other hyperedge models}
\label{tabHyperedges}
\end{table}
\FloatBarrier

\subsection{Testing the Relevance and Coverage criterion}
In this experiment, we analyse the impact of our MRC-based sentence selection step on the content and the Lexical Diversity of the resulting summary. 

The first method, \textit{Greedy Redundancy Removal} (GRR) \cite{xiong2016}, iteratively selects sentences in descending order of scores, provided that the similarity of a newly selected sentence with each already selected sentence does not exceed a threshold $\chi_1\in [0,1]$. The similarity measure is the cosine similarity between tfisf representations of sentences. 

The second method, called \textit{One-Per-Hyperedge} (OPH) method, selects one sentence per topic (i.e. hyperedges) as suggested in \cite{gong2001}. Hyperedges are first ordered in decreasing order of weight. Then, for each hyperedge $e$, the sentence $i$ with maximal associated probability $\psi_{ei}$ is included in the summary.

The third method, referred to as \textit{Maximal Relevance Minimum Similarity} (MRMS) method \cite{yin2015}, seeks a summary maximizing the function
\begin{equation}
Q(S)=\chi_2\underset{i\in S}{\sum}r_i^2-\underset{i,j\in S}{\sum}r_i\text{Sim}(i,j)r_j
\end{equation}
subject to a cardinality constraint $|S|=k$ and with $\chi_2\geq 2$ and a set of relevance scores $\{r_i:1\leq i\leq N_s\}$. We define similarities based on the transition probabilities over our fuzzy hypergraph $\text{Sim}(i,j)=\frac{1}{2}(p(i|j)+p(j|i))$ (equation \ref{sentenceTransition}). The first term of $Q$ enforces the sentence relevance and the second term enforces the Lexical Diversity of the summary. As $Q$ is submodular and non-decreasing, \cite{yin2015} provides an iterative algorithm to find an approximate solution to the problem.

The fourth method, referred to as \textit{Maximum Corpus Similarity} (MCS) \cite{lin2010}, seeks a summary $S$ maximizing
\begin{equation}
R(S)=\underset{i\in V\setminus S}{\sum}\underset{j\in S}{\sum}\text{Sim}(i,j)-\chi_3\underset{i,j\in S}{\sum} \text{Sim}(i,j)
\end{equation}
subject to a capacity constraint and with $\chi_3>0$ and similarities defined in the same way as for MCS algorithm. An iterative algorithm is formulated in \cite{lin2010} to find an approximate solution to the problem.

Our approach shares some similarities with both MRMS (maximum Relevance) and MCS (maximum Coverage). Indeed, we combine both the relevance of sentences and the coverage of topics in our objective function, but we do not impose any constraint on the dissimilarity between selected sentences.

For $\chi_1\in [0,1]$, $\chi_2\in [2,10]$ and $\chi_3\in [0,10]$, the values achieving the best performance based on cross-validation are $\chi_1=0.1$, $\chi_2=3$ and $\chi_3=4.2$. Table \ref{tabCover} displays ROUGE-2, ROUGE-SU4 and Lexical Diversity scores achieved on DUC07 and the corresponding $95\%$ confidence intervals. In terms of ROUGE-SU4, our MRC algorithm outperforms other approaches by at least $7\%$. OPH ($21\%$) yields the worst performance. This confirms that a naive approach selecting one sentence only per hyperedge severely deteriorates the quality of the summary. The Lexical Diversity achieved by our MRC algorithm exceeds that of GRR and MRMS approaches by about $1\%$. The lexical diversity score is higher for MCS method than for our MRC algorithm which was expected since MCS selects lexically dissimilar sentences while our MRC algorithm focuses on Topical Coverage. Moreover, the fact that MCS algorithm achieves a lower ROUGE-SU4 score by $17\%$ compared to our MRC algorithm proves that our topical approach results in a better content coverage than methods focusing on the removal of lexical redundancies. The Lexical Diversity is also higher for OPH which selects one sentence per hyperedge regardless of its centrality in the hypergraph. Nevertheless, this approach is outperformed by $21\%$ by our MRC algorithm in terms of ROUGE-SU4 score.

\begin{table}[!h]
\begin{center}
\resizebox{\textwidth}{!}{\fontsize{5}{3}\selectfont
\begin{tabular}{|| c | c | c | c | c | c ||}
   \hline
   \rule{0pt}{2ex} Sentence Selection Method & ROUGE-2 & ROUGE-SU4 & Lexical Diversity\\
   \hline
	\rule{0pt}{2ex} \textbf{MRC} & $\mathbf{0.12745 (0.11791 - 0.13699)}$ & $\mathbf{0.1792 (0.17065-0.18775)}$ & $0.86313 (0.84105-0.88521)$\\
	GRR & $0.11858 (0.10694 - 0.13021)$ & $0.1682 (0.1603 - 0.1761)$ & $0.85114 (0.81745 - 0.88482)$\\
	OPH & $0.09346 (0.08096 - 0.10595)$ & $0.14857 (0.14135 - 0.15579)$ & $\mathbf{0.95309 (0.94411 - 0.96206)}$\\
	MRMS & $0.12621 (0.11438 - 0.13803)$ & $0.16936 (0.16147 - 0.17725)$ & $0.85403 (0.82349 - 0.88456)$\\
	MCS & $0.10608 (0.0934 - 0.11875)$ & $0.15269 (0.14337 - 0.16201)$ & $0.93929 (0.92726 - 0.95132)$\\ 
   \hline
     
\end{tabular}}
\end{center}
\caption{Performance of our MRC sentence selection compared to GRR, OPH, MRMS and MCS.}
\label{tabCover}
\end{table}
\FloatBarrier

\subsection{Comparison with other graph-based summarization algorithms}

We compare our MRC algorithm to four state-of-the-art graph or hypergraph-based summarizers. Unless stated otherwise, lexical similarity denotes the cosine similarity between tfisf representations of sentences as defined in \cite{blake2006}.

\textit{Topic-sensitive LexRank} (TS-LexRank) defines a graph in which an edge connects two sentences if they have nonzero lexical similarity \cite{R17}. Sentence scores are obtained through a query-biased PageRank algorithm: the score $r_i$ of sentence $i$ is
\begin{equation}
r_i=\omega_1\frac{\text{sim}(i,q)}{\sum_j\text{sim}(j,q)}+(1-\omega_1)\frac{\sum_{j\neq i}\text{sim}(i,j)r_j}{\sum_{l\neq i}\text{sim}(i,l)}
\end{equation}  
in which $\omega_1\in ]0,1[$ is a parameter whose value is set to $0.95$, as in \cite{R17}.

The second method \cite{wan2008}, based on \textit{Hubs and Authorities} algorithm, first discovers sentence clusters by applying agglomerative clustering to tfisf representations of sentences. A bipartite graph is then formed in which sentences and clusters represent vertices and edges have weights corresponding to their lexical similarities. HITS algorithm is then applied to rank both sentences (considered as authorities) and clusters (considered as hubs) based on the iterative formulas
\begin{equation}
r_i = \underset{l}{\sum} \text{sim}(i,l)q_l\text{, }\text{ }\text{ }\text{ }\text{ }q_l = \underset{i}{\sum} \text{sim}(i,l)r_i
\end{equation}
where $r_i$ is the score of $i$-th sentence and $q_l$ is the score of $l$-th cluster. To produce query-oriented summaries, we restrict the sentence set to the top $10\%$ of sentences relevant to the query, as suggested in \cite{wanng2013}.

\textit{HyperSum} is a hypergraph-based text summarizer \cite{wanng2013}. It first applies DBSCAN algorithm to detect clusters of lexically similar sentences. A hypergraph is built in which each cluster defines a hyperedge connecting sentences of the cluster. Sentence scores are obtained by applying a semi-supervised learning algorithm in which query relevance scores are propagated across the hypergraph.

\textit{HERF} builds on a similar principle but it includes an initial topic modelling step in which topics are extracted from sentences using a topic model \cite{xiong2016}. DBSCAN clustering algorithm is then applied to topic representations of sentences in order to extract sentence clusters. A hypergraph is built in the same way as for HyperSum. Scores are computed by applying a diversified version of PageRank algorithm called DivRank, which extracts both relevant and non-redundant sentences. The value of the DivRank's transition factor is set to $0.97$ as in \cite{xiong2016}.

Table \ref{tabGraphMethods} displays ROUGE-2 and ROUGE-SU4 scores for all five methods. We observe that our MRC algorithm outperforms TS-LexRank and Hubs and Authorities by at least $8\%$ on DUC06 and DUC07 and at least $2\%$ on DUC05 which justifies our use of a hypergraph that incorporates group relationships among sentences rather than a simple graph. HyperSum performs slightly better than MRC on DUC05 in terms of ROUGE-2. However, our method outperforms HyperSum and HERF by at least $5\%$ on DUC06 and DUC07. These two hypergraph approaches are limited to the detection of disjoint sentence clusters and do not take advantage of the fuzzy semantic relationships between sentences. They also fail to provide a proper method of sentence selection after sentence ranking, while our method involves the maximization of Relevance and Topical Coverage.

\begin{table}[!h]
\begin{center}
\resizebox{\textwidth}{!}{\fontsize{5}{5}\selectfont
\begin{tabular}{|| c | c | c | c | c | c | c ||}
   \hline
	\rule{0pt}{2ex}  & \multicolumn{2}{c|}{DUC05} & \multicolumn{2}{c|}{DUC06} & \multicolumn{2}{c||}{DUC07}\\    
   \hline
   \rule{0pt}{2ex} Algorithm & ROUGE-2 & ROUGE-SU4 & ROUGE-2 & ROUGE-SU4 & ROUGE-2 & ROUGE-SU4\\
   \hline
	\rule{0pt}{2ex} \textbf{MRC} & $\mathbf{0.07864}$ & $0.12824$&$\mathbf{0.10947}$ & $\mathbf{0.16141}$&$\mathbf{0.12745}$ & $\mathbf{0.17920}$\\
	TS-LEXRANK & $0.07231$ & $0.12554$&$0.08892$ & $0.14741$& $0.11048$ & $0.16524$\\
	HUBS \& AUTH. & $0.06902$ & $0.12217$& $0.08172$ & $0.13731$ & $0.10493$ & $0.15756$\\
	HYPERSUM & $0.07291$ & $\mathbf{0.13087}$& $0.09569$ & $0.15182$ & $0.11197$ & $0.16612$\\
	HERF & $0.06212$ & $0.12244$& $0.07226$ & $0.15346$ & $0.11234$ & $0.16330$\\
   \hline   
   
\end{tabular}}
\end{center}
\caption{Comparison of our MRC algorithm with four methods on DUC05, DUC06 and DUC07.}
\label{tabGraphMethods}
\end{table}
\FloatBarrier

\subsection{Comparison with DUC systems}
Finally, we compare the performance of our method to that of other summarizers submitted for DUC07 summarization tasks. Regarding DUC07 question answering task, table \ref{tabCompar} reports ROUGE-2 and ROUGE-SU4 for the top four systems ($S15$, $S29$, $S4$, $S24$), for the worst human summarizer ($Hum$), for the baseline chosen by NIST (leading sentences of randomly selected documents) and for the average performance of all systems. The same results are displayed for DUC06 dataset in which the best systems are $S24$, $S15$, $S12$ and $S8$, and for DUC05 in which the best systems are $S15$, $S17$, $S10$ and $S8$. Apart from DUC05, we observe that our proposed method slightly outperforms other summarizers in terms of ROUGE-2 and ROUGE-SU4 but it performs worse than the human summaries which was expected since we merely extract sentences from the original corpus, hence the resulting summary cannot match the quality of abstractive summaries produced by humans. Overall, we observe that our system achieves better performances on DUC06 and DUC07 than it does on DUC05 dataset.

\begin{table}[!h]
\begin{center}
\resizebox{\textwidth}{!}{\fontsize{5}{3}\selectfont
\begin{tabular}{|| c | c | c | c | c | c | c ||}
   \hline
	\rule{0pt}{2ex}  & \multicolumn{2}{c|}{DUC05} & \multicolumn{2}{c|}{DUC06} & \multicolumn{2}{c||}{DUC07}\\  
   \hline
   \rule{0pt}{2ex} Method & ROUGE-2& ROUGE-SU4 & ROUGE-2& ROUGE-SU4 & ROUGE-2& ROUGE-SU4\\
   \hline
	\rule{0pt}{2ex} Hum & $0.0897$ & $0.151$ & $0.13260$ & $0.18385$ & $0.17528$ & $0.21892$\\
	\rule{0pt}{2ex} \textbf{MRC} & $\mathbf{0.07864}$ & $0.12824$& $\mathbf{0.10947}$ & $\mathbf{0.16141}$& $\mathbf{0.12745}$ & $\mathbf{0.1792}$\\
1st & $0.07251$ & $\mathbf{0.13163}$ & $0.09558$ & $0.15529$ & $0.12448$ & $0.17711$\\
2nd & $0.07174$ & $0.12972$ & $0.09097$ & $0.14733$ & $0.12028$ & $0.17074$\\
3rd & $0.06984$ & $0.12525$ & $0.08987$ & $0.14755$ & $0.11887$& $0.16999$\\
4th & $0.06963$ & $0.12795$ & $0.08954$ & $0.14607$ & $0.11793$ & $0.17593$\\
Syst. Av.& $0.05842$ & $0.11205$ & $0.07463$ & $0.13021$ & $0.09597$&$0.14884$\\
Basel. & $0.04026$ & $0.08716$ & $0.04947$ & $0.09788$ & $0.06039$&$0.10507$\\
\hline
\end{tabular}}
\end{center}
\caption{Comparison with DUC05, DUC06 and DUC07 systems}
\label{tabCompar}
\end{table}
\FloatBarrier

\section{Conclusion}
In this paper, we proposed a novel query-oriented summarization approach which extracts important and query-relevant sentences of a corpus based on the definition of a fuzzy hypergraph over sentences. Existing graph and hypergraph-based summarizers rely on lexical similarities between sentences, namely relationships of term co-occurrences, which fail to capture semantic similarities. We propose a new system in which semantic relationships between sentences are captured by a probabilistic topic model. The resulting topics are modelled as hyperedges of a fuzzy hypergraph in which nodes are sentences. Sentences are then scored based on their relevance to the query and their centrality in the hypergraph using a fuzzy hypergraph extension of personalized PageRank algorithm. Then, a set of sentences is selected by simultaneously maximizing individual Relevance scores and joint Topical Coverage, which encourages the topical diversity of the resulting summary. Topical Coverage maximization is formulated as a fuzzy extension of dominating set problem. A polynomial time approximation algorithm for sentence selection is provided, based on the theory of submodular functions. The algorithm produces more informative summaries with a better coverage of topics compared to existing systems. Experimental results show that both our topic-based fuzzy hypergraph model and our sentence selection algorithm contribute to an improvement in the content coverage of the summaries, as measured by ROUGE scores. Moreover, a thorough comparative analysis with other graph-based summarizers and summarizers presented at DUC contest demonstrates the superiority of our method in terms of content coverage. As a future research direction, we will investigate how to adapt the model for related tasks including update summarization and community question answering. We will also attempt to incorporate sentence fusion and compression in our fuzzy hypergraph-based method to determine whether topical relationships can help in these tasks.

\end{document}